\documentclass[letterpaper, 10 pt, journal]{IEEEtran}

\IEEEoverridecommandlockouts

\usepackage{graphics} 
\usepackage{epsfig} 
\usepackage{times} 
\usepackage{amsmath} 
\usepackage{stackengine}
\usepackage{amssymb} 
\usepackage{amsthm}
\usepackage{multirow}
\usepackage{upgreek}
\usepackage{lettrine}
\usepackage{subfigure}
\usepackage{cancel}
\usepackage[table]{xcolor}
\usepackage{mathtools,amssymb,lipsum}

\usepackage{cuted}
\usepackage{mathtools}
\usepackage{multirow}  
\usepackage{color}
\usepackage{cite}
\usepackage{lipsum}
\usepackage{cuted}
\usepackage{caption}
\usepackage{blindtext}
\usepackage{mathrsfs}
\usepackage{enumitem}
\usepackage[para]{footmisc}

\newtheorem{thm}{Theorem}

\newtheorem{defn}{Definition}
\newtheorem{cor}{Corollary}
\newtheorem{lem}{Lemma}

\newtheorem{remark}{Remark}
\newtheorem{example}{Example}
\newtheorem{problem}{Problem}

\DeclareMathOperator{\sgn}{sgn}

\title{M\"{o}bius Transformation-Based Circular Motion Control for Unicycle Robots in Nonconcentric Circular Geofences}

\author{Shubham Singh and Anoop Jain
\thanks{The authors are with the Department of Electrical Engineering, Indian Institute of Technology Jodhpur, 342030, India (e-mail: {shubham.1@iitj.ac.in}; {anoopj@iitj.ac.in}).}}

\begin{document}
\maketitle


\begin{abstract}
Nonuniform motion constraints are ubiquitous in robotic applications. Geofencing control is one such paradigm where the motion of a robot must be constrained within a predefined boundary. This paper addresses the problem of stabilizing a unicycle robot around a desired circular orbit while confining its motion within a \emph{nonconcentric} external circular boundary. Our solution approach relies on the concept of the so-called M\"{o}bius transformation that, under certain practical conditions, maps two nonconcentric circles to a pair of concentric circles, and hence, results in uniform spatial motion constraints. The choice of such a M\"{o}bius transformation is governed by the roots of a quadratic equation in the post-design analysis that decides how the regions enclosed by the two circles are mapped onto the two planes. We show that the problem can be formulated either as a \emph{trajectory-constraining problem} or an \emph{obstacle-avoidance problem} in the transformed plane, depending on these roots. Exploiting the idea of the barrier Lyapunov function, we propose a unique control law that solves both these contrasting problems in the transformed plane and renders a solution to the original problem in the actual plane. By relating parameters of two planes under M\"{o}bius transformation and its inverse map, we further establish a connection between the control laws in two planes and determine the control law to be applied in the actual plane. Simulation and experimental results are provided to illustrate the key theoretical developments. 
\end{abstract}

\begin{IEEEkeywords}
M\"{o}bius transformation, geofencing control, nonuniform constraints, circular motion control, unicycle robot, barrier Lyapunov function, stabilization.  
\end{IEEEkeywords}

\section{Introduction}\label{section_1}

\lettrine {D}ynamical systems, in practice, encounter \emph{nonuniform} operational constraints. For instance, robotic systems face several such constraints dictated by their application needs, especially in complex environments like irregular, cluttered, or confined spaces, and it is a challenging problem to synthesize a stabilizing controller for their safe operation. Particularly, in geofencing applications where the workspace is specified via virtual boundaries for a real-world geographic area, the robot's motion is limited by nonuniform spatial constraints, concerning the nature of geofence \cite{hermand2018constrained}. Several critical missions such as surveillance or monitoring near territorial boundaries or over a region affected by fire or flood, ocean explorations, autonomous driving in a lane, etc., are some examples where geofencing control of robotic systems is necessary for safety considerations. While allowing operating in the regime of a geofence, it is equally important that the proposed controller also handles nonholonomic motion constraints of a real robotic system such as UAV (unmanned aerial vehicle) and UGV (unmanned ground vehicle) \cite{lopez2019adaptive}. 

Motivated by these applications, this paper addresses the problem of stabilizing an underactuated unicycle robot around a desired circular orbit while restricting its trajectories within a nonconcentric external circular boundary. This differs from conventional circular motion/formation control problems in the literature \cite{chen2013remark, seyboth2014collective, brinon2015distributed, cao2015uav, yu2016distributed, brinon2019multirobot} where no restriction is imposed on the robots' (or agents') trajectories. Our previous work \cite{jain2019trajectory} in this direction addressed a problem, where the external circular boundary was concentric with the desired stabilizing circular orbit, thereby imposing uniform trajectory-restricting constraints on the robot's motion due to the constant radial distance between the two circles. An extension to this problem for a general class of polar curves was studied in \cite{hegde2023synchronization}, where both the stabilizing curve and the external trajectory-constraining boundary were concentric and had a similar nature. As a result, the radial distance between the two curves was constant throughout the complete rotation. However, in practice, the trajectory-constraining set (or boundary) is not necessarily concentric with the desired stabilizing curve, resulting in nonuniform spatial constraints on the robot's motion. Our consideration of a nonconcentric external circular boundary depicts one such practical scenario and makes the problem challenging. 

In recent years, there has been a growing interest among researchers in control design methodologies that accommodate system and operational constraints, primarily focusing on two solution approaches: control barrier function (CBF)-based methods \cite{ames2016control,dawson2023safe} and barrier Lyapunov function (BLF) based methods \cite{tee2009barrier,liu2018adaptive}. While the former involves solving a quadratic programming (QP) problem for control design \cite{desai2023auxiliary}, the latter is relatively simpler and extends the concept of conventional control Lyapunov function (CLF) to systems with constraints \cite{romdlony2016stabilization}. Various BLFs have been proposed in the literature to address a range of physical and operational constraints while achieving desired control objectives. These include re-centered barrier function \cite{panagou2015distributed}, parametric barrier function \cite{han2019robust}, universal barrier functions \cite{jin2018adaptive}, tangent-type BLF \cite{tang2013tangent}, integral-type BLF \cite{tee2012control}, among others. In this paper, we adopt the logarithmic-type BLF, introduced in \cite{tee2009barrier}, because of its simplicity in handling the constraints. 

The primary challenge in applying BLF-based approaches lies in appropriately designing the so-called barrier functions in their formulation that must accurately characterize the desired constraint requirements of a system or mission. In the literature, two types of barrier functions are commonly discussed: symmetric BLFs and asymmetric BLFs, addressing symmetric and asymmetric system constraints, respectively \cite{tee2009barrier,jin2018adaptive}. When focusing on meeting prescribed transient and steady-state specifications of the system's response, constraints can be formulated as time-varying barrier functions of output \cite{verginis2017robust}. When aiming to satisfy certain requirements on the system's states (e.g., the position of a robot in geofencing applications), the barrier functions must be derived considering the constraining-set properties. However, formulating such time-varying barrier functions becomes challenging in arbitrary geofencing regions, where motion constraints exhibit nonuniform spatial dependence and are not solely time-dependent, in general, \cite{singletary2022onboard,garg2019control,meng2020trajectory,jang2024safe}.

As a remedy, our approach involves transforming the proposed problem to a new coordinate system where spatial constraints remain symmetrical at all times. In this direction, we employ the concept of M\"{o}bius transformation \cite{priestley2003introduction,needham2023visual} to convert two nonconcentric circles into a pair of concentric circles, thereby imposing uniform spatial constraints on the robot's motion. We emphasize that it is not only about the mapping of the robot's position under M\"{o}bius transformation, instead, it is about the motion mapping of the unicycle robot where its velocity and control parameters are also mapped to the transformed plane. Consequently, our strategy involves designing the controller in the transformed plane first and then mapping it back to the actual plane to design the actual controller. Since our main motive in this paper is to establish a connection between robot motion in the actual plane and the transformed plane under M\"{o}bius transformation, our discussion is focused on the single robot case, which can be generalized to a multi-robot setting for solving desired cooperative control problems.

\paragraph*{Summary of Contributions}
Building upon the idea of M\"{o}bius transformation and leveraging its symmetry and circle preserving properties, we propose an appropriate M\"{o}bius transformation $f: z \mapsto w$ that maps two nonconcentric circles $\mathcal{C}$ and $\mathcal{C}'$ in the $z$-plane to the concentric circles $f(\mathcal{C})$ and $f(\mathcal{C}')$ in the $w$-plane, respectively. We show that the existence of such a M\"{o}bius transformation relies on certain natural conditions aligned with practical requirements for the problem at hand (such as non-touching, non-intersecting, etc.) on the circles $\mathcal{C}$ and $\mathcal{C}'$ and the (real) roots $\alpha_s$ and $\alpha_{\ell}$ of a quadratic equation that characterize the mapping $f$. We rigorously analyze how the regions enclosed between the circles $\mathcal{C}$ and $\mathcal{C}'$ are mapped under $f$, considering the roots $\alpha_s$ and $\alpha_{\ell}$ where $\alpha_s$ is the smaller root and $\alpha_{\ell}$ is the larger root in the absolute sense. It is shown that $\alpha_s$ preserves the interior-exterior relationship of the enclosed regions between the two curves in both planes, while $\alpha_{\ell}$ reverses it. Consequently, the considered problem can be formulated either as a \emph{trajectory-constraining problem} or an \emph{obstacle-avoidance problem} in the transformed plane, corresponding to the roots $\alpha_s$ and $\alpha_{\ell}$, respectively. However, because of the constant radial separation between the circles $f(\mathcal{C})$ and $f(\mathcal{C}')$, the robot's motion is subject to uniform spatial constraints in the transformed plane. By exploiting the idea of logarithmic BLF, our second contribution is towards deriving a unique control law $\Omega$ in the transformed plane that simultaneously solves both the aforementioned problems corresponding to $\alpha_s$ and $\alpha_{\ell}$. We further show that this transformed control $\Omega$ essentially solves the original problem in the actual plane. We also obtain conservative bounds on the post-design signals in both transformed and actual planes and prove their uniform boundedness. Our third contribution is towards establishing a connection between the two robot models in the two planes and determining the control law $\omega$ applied to the robot in the actual plane from the control $\Omega$ in the transformed plane. We accomplish this objective by exploring several intermediate relations connecting different parameters in two planes. Due to resemblance and their significance in implementation, we examine the connection between the two planes while considering both the M\"{o}bius transformation $f$ and its inverse $f^{-1}$, whose existence is also proved for the considered problem. Through simulations and experiments conducted on the Khepera IV ground robot, we validate the proposed approach.

\paragraph*{Paper Structure}
Section~\ref{section_2} provides a brief overview of the M\"{o}bius transformation and investigates its properties in transforming two nonconcentric circles into concentric circles and the regions enclosing them. In Section~\ref{section_3}, we present the system model and formulate the problem in the actual plane. Section~\ref{section_4} outlines the equivalent problems in the transformed plane and devises the transformed controller using BLF. The proposed controller effectively addresses the original problem in the actual plane, which is also discussed in this section. Section~\ref{section_5} establishes several crucial relationships connecting parameters in the actual and transformed planes using both the M\"{o}bius transformation and its inverse and subsequently links the control laws in two planes. Section~\ref{section_6} derives conservative bounds on the post-design signals using BLF. Simulation and experimental results are presented in Section~\ref{section_7}, followed by concluding remarks in Section~\ref{section_8}. Two appendices are included at the end.

\paragraph*{Preliminaries}
The set of real numbers is denoted by $\mathbb{R}$, non-negative real numbers by $\mathbb{R}_{+}$, complex numbers by $\mathbb{C}$, natural numbers by $\mathbb{N}$, and integers by $\mathbb{Z}$. The symbol $|\bullet|$ denotes the absolute value of the real/complex number $\bullet$, and $\|\star\|$ represents the Euclidean-norm of the $n$-dimensional real vector $\star \in \mathbb{R}^n$. For any complex number $z \in \mathbb{C}$, $\Re(z)$ and $\Im(z)$ represent its real and imaginary parts, respectively. The imaginary unit is denoted by $i \coloneqq \sqrt{-1}$. The inner product $\langle z_1, z_2 \rangle$ of two complex numbers $z_1, z_2 \in \mathbb{C}$ is defined as $\langle z_1, z_2 \rangle \coloneqq \Re(z^{\ast}_1 z_2)$, where $z^\ast$ represents the complex conjugate of $z \in \mathbb{C}$.  Let $\mathcal{D} \subseteq \mathbb{R}^N$ and $\mathfrak{f}: \mathcal{D} \to \mathbb{R}$ be a differential map. The gradient of $\mathfrak{f}$ is denoted by $\nabla \mathfrak{f}$. The abbreviation ``$\text{mod}$" denotes the modulo operation and $\sgn(\bullet)$ represents the signum function of $\bullet: \mathbb{R} \to \mathbb{R}$. The unit circle, or the one dimensional unit sphere, in the complex plane $\mathbb{C}$ is denoted by $\mathbb{S}^1$. For clarity, the time argument ``$t$" is often suppressed when clear from the context.  

To facilitate the subsequent analysis and substantiate various results in this paper, we state the following definition and a convergence lemma from \cite{tee2009barrier}, characterizing the properties of BLF.

\begin{defn}[\hspace{-.1pt}Barrier Lyapunov Function \cite{tee2009barrier}]\label{def_BLF}
	A Barrier Lyapunov Function is a scalar function $V(\mathcal{X})$ of state vector $\mathcal{X} \in \mathcal{D}$ of the system $\dot{\mathcal{X}} = \mathfrak{f}(\mathcal{X})$ on an open region $\mathcal{D}$ containing the origin, that is continuous, positive-definite, has continuous first-order partial derivatives at every point of $\mathcal{D}$, has the property $V(\mathcal{X}) \to \infty$ as $\mathcal{X}$ approaches the boundary of $\mathcal{D}$, and satisfies $V(\mathcal{X}(t)) \leq \varphi_0, \forall t\geq 0$, along the solution of $\dot{\mathcal{X}} = \mathfrak{f}(\mathcal{X})$ for $\mathcal{X}(0) \in \mathcal{D}$ and some positive constant $\varphi_0$.
\end{defn}

\begin{lem}[\hspace{-.1pt}Convergence with BLF{\cite[pp. 919$-$920]{tee2009barrier}}]\label{lem_BLF}
	For any positive constant $\varpi$, let $\Lambda \coloneqq \{\zeta \in \mathbb{R} \mid |\zeta| < \varpi\} \subset \mathbb{R}$ and $\mathcal{N} \coloneqq \mathbb{R}^\ell  \times \Lambda \subset \mathbb{R}^{\ell+1}$ be open sets. Consider the system $\dot{{\mathcal{X}}} = {h}(t, \mathcal{X})$,  where, $\mathcal{X} \coloneqq [w, \zeta]^\top \in \mathcal{N}$, and ${h} : \mathbb{R}_+ \times \mathcal{N} \to \mathbb{R}^{\ell+1}$ is piecewise continuous in $t$ and locally Lipschitz in $\mathcal{X}$, on $\mathbb{R}_+ \times \mathcal{N}$. Suppose there exist functions $U: \mathbb{R}^\ell \times \mathbb{R}_{+} \to \mathbb{R}_+$ and $V_1: \Lambda \to \mathbb{R}_+$, continuously differentiable and positive definite in their respective domains, such that $ V_1(\zeta) \to \infty \ \text{as} \ |\zeta| \to \varpi$ and $\varphi_1(\|{w}\|) \leq U({w}, t) \leq \varphi_2(\|{w}\|)$, where, $\varphi_1$ and $\varphi_2$ are class $\mathcal{K}_\infty$ functions. Let $V(\mathcal{X}) \coloneqq V_1(\zeta) + U(w, t)$, and $\zeta(0) \in \Lambda$. If it holds that $\dot{V} = (\nabla V)^\top{{h}} \leq 0$, in the set $\zeta \in \Lambda$, then $\zeta(t) \in \Lambda, \forall t \in [0, \infty)$.
\end{lem}

\section{A Review of M\"{o}bius Transformation}\label{section_2}

\subsection{The M\"{o}bius Transformation}
A M\"{o}bius transformation is a mapping of the form \cite[Chapter 2, pg. 23]{priestley2003introduction}, \cite[Chapter 3, pg. 137]{needham2023visual}:
\begin{equation}\label{mt}
	z \mapsto w = f(z) \coloneqq \frac{az+b}{cz+d},
\end{equation}
where $a, b, c, d \in \mathbb{C}$, and $ad - bc \neq 0$. The condition $ad - bc \neq 0$ ensures that $f$ is neither undefined (i.e., at least one of $c$ and $d$ is nonzero) nor one that is identically a constant, as can be seen in the case when $c \neq 0$,
\begin{equation*}
	f(z) = \frac{a}{c} - \frac{ad - bc}{c(cz + d)},
\end{equation*}
which maps to a constant $a/c$ if $ad - bc = 0$. The special case of $c = 0$ implies that neither $a$ nor $d$ is zero, and hence, $f(z)$ is well-defined. The quantity $ad-bc$ is called the determinant of $f$. Clearly, the domain of $f$ is $\mathbb{C} \setminus \{-d/c\}$ if $c \neq 0$, and $\mathbb{C}$ if $c = 0$. For all $z \in \mathbb{C} \setminus \{-d/c\}$, the derivative of $f$ is given by \cite[Chapter 8, pg. 95]{priestley2003introduction}:
\begin{equation}\label{mobius_transform_derivative}
	f'(z) = \frac{ad - bc}{(cz + d)^2} \neq 0,	
\end{equation}
since $ad - bc \neq 0$. Hence, $f$ is conformal in $\mathbb{C} \setminus \{-d/c\}$, where $c \neq 0$. In the special case of $c = 0$, $f(z) = (az+b)/d$ is a linear mapping that is conformal everywhere. In the subsequent discussion, we do not explicitly mention the special case of $c =0$, as it is straightforward. Recall from complex analysis that conformal maps are holomorphic (or analytic). Further, solving for $z$ in \eqref{mt}, the inverse of $f$ is obtained as \cite[Chapter 2, pg. 23]{priestley2003introduction}:
\begin{equation}\label{imt}
	w \mapsto z = f^{-1}(w) = \frac{dw-b}{a-cw},
\end{equation}
which has the same determinant $ad - bc \neq 0$ that of $f$, and hence, is also a M\"{o}bius transformation. The domain of $f^{-1}$ is $\mathbb{C} \setminus \{a/c\}$. Therefore, $f(z) : \mathbb{C} \setminus \{-d/c\} \mapsto \mathbb{C} \setminus \{a/c\}$ defines a bijective and holomorphic map. For simplicity, the M\"{o}bius transformation $f$ is often regarded as a mapping from $\mathbb{C}_{\infty}$ to $\mathbb{C}_{\infty}$ such that $f(\infty) = a/c$ and $f(-d/c) = \infty$ if $c \neq 0$, and $f(\infty) = \infty$ if $c = 0$, in the extended complex plane $\mathbb{C}_{\infty} = \mathbb{C} \cup \{\infty\}$ (i.e., the complex plane augmented by the point at infinity). Notably, this extension accounts for the ``inversion mapping" $z \mapsto 1/z$ on $\mathbb{C} \setminus \{0\}$ to a mapping from $\mathbb{C}_{\infty}$ to $\mathbb{C}_{\infty}$ by the assignments $0 \mapsto \infty$ and $\infty \mapsto 0$. With the aid of $\mathbb{C}_{\infty}$, it can be deduced that $f: \mathbb{C}_{\infty} \to \mathbb{C}_{\infty}$ is bijective and conformal everywhere in its domain \cite[Chapter 2, pg. 23]{priestley2003introduction}. 

\subsection{Mapping of Two Nonconcentric Circles to Concentric Circles}
The M\"{o}bius transformation \eqref{mt} shares some remarkable properties \cite[Chapter 3, pg. 168]{needham2023visual}: (i) it maps circles to circles (a straight line is viewed as a degenerate circle through infinity), (ii) If two points are symmetric with respect to a circle, then their images under \eqref{mt} are symmetric with respect to the image circle. This is known as the \emph{symmetry principle}. Note that two points $z$ and $\tilde{z}$ are symmetric with respect to the unit circle if they are related by the inversion mapping $1/\bar{z}$ such that the product of their distances from the center of the circle is unity. This concept can be generalized to a circle with a non-unity radius where this product equals the square of the radius of the circle. For further details, readers are referred to \cite[Chapter 3, pp. 139$-$142]{needham2023visual}. By leveraging these properties, one can devise a suitable M\"{o}bius transformation that maps any two non-intersecting, nonconcentric circles in the $z$-plane to two concentric circles in the $w$-plane. In this direction, we state the following theorem from \cite[Chapter 16, pp. 1234$-$1240]{turyn2014advanced} and provide its proof for clarity of further discussion. 

\begin{thm}[Mapping of nonconcentric circles to concentric circles]\label{thm_main_theorem}
Let $\mathcal{C}: |z| = 1$ and $\mathcal{C}': |z - \lambda| = \mu$ be two circles, where $\lambda \neq 0$ and $\mu > 0$ are given real numbers. Suppose $\mathcal{C}$ and $\mathcal{C}'$ have no point in common (i.e., the two circles do not intersect or touch each other). Then, we can find real numbers $\alpha$ and $\beta$ such that the M\"{o}bius transformation 
\begin{equation}\label{mobius_transformation}
	w = f(z) = \frac{z+\alpha}{z+\beta},
\end{equation}
maps both the circles $\mathcal{C}$ and $\mathcal{C}'$ to concentric circles centered at $0$ in the $w$-plane, as long as, it turns out that $z = -\beta$ is on neither $\mathcal{C}$ nor $\mathcal{C}'$. Here, $\beta = 1/\alpha$, where the real number $\alpha$ satisfies
\begin{equation}\label{mobius_roots}
	\lambda\alpha^2 + (\lambda^2 - \mu^2 + 1)\alpha + \lambda = 0.
\end{equation}
Further, the images of the circles in the $w$-plane are given by:
\begin{equation}\label{circle_images}
	f(\mathcal{C}): |w| = |\alpha|, \qquad f(\mathcal{C}'): |w| = \left|\frac{\lambda + \alpha}{\mu}\right|.
\end{equation}
\end{thm} 

Without loss of generality, and if necessary, by a suitable coordinate transformation as stated in Remark~\ref{remark_coordinate_transformation} below, the centers of both circles $\mathcal{C}$ and $\mathcal{C}'$ in Theorem~\ref{thm_main_theorem} are assumed to be located on the real axis. Before proving Theorem~\ref{thm_main_theorem}, we first present the following lemma regarding the non-collinearity of three points, followed by a corollary characterizing the properties of the roots of quadratic equation \eqref{mobius_roots}.

\begin{lem}[\hspace{-.1pt}Non-collinear points {\cite[Chapter 16, pp. 1234]{turyn2014advanced}}]\label{lem_non-collinier points}
If $w_1$, $w_2$ and $w_3$ are three non-collinear points, then there is exactly one circle that passes through $w_1$, $w_2$ and $w_3$.
\end{lem}

\begin{cor}[Roots of {\eqref{mobius_roots}}]\label{cor_quadractic_roots}
Under the conditions specified in Theorem~\ref{thm_main_theorem}, the roots $\alpha$ of the quadratic equation \eqref{mobius_roots} satisfy the following properties:	
\begin{enumerate}[leftmargin=*]
\item[(a)] The product of roots is unity. 
\item[(b)] The condition ``the point $z = -\beta$ lies neither on $\mathcal{C}$ nor $\mathcal{C}'$" holds if and only if the solutions of \eqref{mobius_roots} are not $\alpha = \pm 1$. 
\end{enumerate}
\end{cor} 

\begin{proof}
The statements are proven sequentially as follows: 
	\begin{enumerate}[leftmargin=*]
		\item[(a)] Since $\lambda \neq 0$ in Theorem~\ref{thm_main_theorem}, \eqref{mobius_roots} can also be expressed as 
		\begin{equation}\label{mobius_rrot_normalized}
			\alpha^2 + \left(\frac{\lambda^2 - \mu^2 + 1}{\lambda}\right)\alpha + 1 = 0,
		\end{equation}	
		implying that the product of roots of \eqref{mobius_roots} is unity. Consequently, both roots share the same sign (either positive or negative), and neither of them is zero, i.e., $\alpha \neq 0$.
		  	
		\item[(b)] {\emph{(Necessity).}} We prove it by contradiction. Assume that $z = -\beta$ lies on $\mathcal{C}$, i.e., $|\beta| = 1$. This implies that $\alpha = \pm 1$ since $\beta = 1/\alpha \in \mathbb{R}$, which contradicts our hypothesis. Further, assume that $z = -\beta$ lies on $\mathcal{C}'$, i.e., $|\beta + \lambda| = \mu$. Substituting $\beta = 1/\alpha$, we obtain $|\lambda + \frac{1}{\alpha}| = \mu$. Squaring both sides, we get $\lambda^2 - \mu^2 = -[(1+2\lambda \alpha)/\alpha^2]$, which upon substitution into \eqref{mobius_roots}, yields $(\alpha^2 - 1)(\lambda + \frac{1}{\alpha}) = 0$, implying $\alpha = \pm 1$ since $|\lambda + \frac{1}{\alpha}| = \mu$. This again contradicts our hypothesis. \\		
		
		{\emph{(Sufficiency).}} 
		If $\alpha \neq \pm 1$, then $|\beta| \neq 1$ (since $\beta = 1/\alpha$), i.e., $z = \pm \beta$ does not lie on $\mathcal{C}$. Moreover, by substituting $\alpha$ into \eqref{mobius_roots}, it is evident that $|\lambda \pm 1| \neq \mu \implies |\lambda + \beta| \neq \mu$ (where the $+$ sign corresponds to $\alpha \neq +1$ and the $-$ sign to $\alpha \neq -1$). This implies that $z = -\beta$ does not lie on $\mathcal{C}'$. It's noteworthy that only the point $z = -\beta$ is not on $\mathcal{C}'$, unlike the points $z = \pm \beta$ in the case of $\mathcal{C}$.		
\end{enumerate}
This concludes the proof. 
\end{proof}

Now, we are ready to proceed with the proof of Theorem~\ref{thm_main_theorem}. 

\begin{proof}[Proof of Theorem~\ref{thm_main_theorem}]
To prove that the image of $\mathcal{C}: |z| = 1$, under \eqref{mobius_transformation}, is $f(\mathcal{C}): |w| = |\alpha|$, we start by obtaining
		\begin{equation*}
			|w|^2 {=} ww^\ast {=} \left[\frac{z+\alpha}{z+\beta}\right] \left[\frac{z^\ast+\alpha}{z^\ast+\beta}\right] {=} \frac{|z|^2 + \alpha(z+z^\ast) + \alpha^2}{|z|^2 + \beta(z+z^\ast) + \beta^2}.
		\end{equation*}
		Substituting $|z| = 1$ and $\beta = 1/\alpha$, we get $|w|^2 = \alpha^2$, implying $|w| = |\alpha|$. 
		
		To prove that the image of $\mathcal{C}': |z - \lambda| = \mu$, under \eqref{mobius_transformation}, is $f(\mathcal{C}'): |w| = \left|{(\lambda + \alpha)}/{\mu}\right|$, we use Lemma~\ref{lem_non-collinier points}. It suffices to show that $f(z)$ maps three distinct points on the circle $\mathcal{C}'$ to three distinct points on the circle $|w| = \left|{(\lambda + \alpha)}/{\mu}\right|$. In this direction, we consider three points $z = \lambda + \mu$ and $z = \lambda \pm i\mu$ that are distinct and lie on the circle $\mathcal{C}'$. Firstly, we show that (i) $|f(\lambda + \mu)| = |f(\lambda \pm i\mu)| = \left|{(\lambda + \alpha)}/{\mu}\right|$, and (ii) $f(\lambda +\mu), f(\lambda + i\mu), f(\lambda - i\mu)$ are distinct points. Using \eqref{mobius_transformation}, 
		\begin{equation*}
			|f(\lambda \pm i\mu)| = \left|\frac{\lambda \pm i\mu+\alpha}{\lambda \pm i\mu+\beta}\right| = \left|\frac{(\lambda + \alpha) \pm i\mu}{(\lambda + \beta) \pm i \mu}\right|.
		\end{equation*} 
		Multiplying both numerator and denominator by $\mu$ gives:
		\begin{equation}\label{colliniear_points}
			|f(\lambda \pm i\mu)| = \left|\frac{[(\lambda + \alpha) \pm i\mu]\mu}{[(\lambda + \beta) \pm i \mu]\mu}\right| = \left|\frac{(\lambda + \alpha)\mu \pm i\mu^2}{(\lambda + \beta)\mu \pm i \mu^2}\right|. 
		\end{equation} 
		Dividing \eqref{mobius_roots} by $\alpha$ (as $\alpha \neq 0$, according to Corollary~\ref{cor_quadractic_roots}) and using $\beta = 1/\alpha$, one can obtain
		\begin{equation}\label{mu_square}	
			\mu^2 = (\lambda + \alpha)(\lambda + \beta).
		\end{equation}
		Substituting \eqref{mu_square} into \eqref{colliniear_points}, we simplify to:
		\begin{align*}
			|f(\lambda \pm i\mu)| &= \left|\frac{(\lambda + \alpha)\mu \pm i(\lambda + \alpha)(\lambda + \beta)}{(\lambda + \beta)\mu \pm i \mu^2}\right| \\
			& = \left|\frac{\lambda + \alpha}{\mu}\right|\left|\frac{\mu \pm i(\lambda + \beta)}{(\lambda + \beta) \pm i \mu}\right| = \left|\frac{\lambda + \alpha}{\mu}\right|, 
		\end{align*}
		since $|\mu \pm i(\lambda + \beta)| = |(\lambda + \beta) \pm i \mu| = \sqrt{(\lambda + \beta)^2 + \mu^2}$. Similarly, it is easy to show that $|f(\lambda + \mu)| = \left|{(\lambda + \alpha)}/{\mu}\right|$. This meets the first requirement. Now, it remains to show that these points are distinct. We prove this by contradiction. Suppose if 
		\begin{equation*}
			f(\lambda + i\mu) = f(\lambda - i\mu) \Longleftrightarrow \alpha = \beta \Longleftrightarrow \alpha = \pm 1,
		\end{equation*} 
		which is impossible according to Corollary~\ref{cor_quadractic_roots}, and hence, $f(\lambda + i\mu) \neq f(\lambda - i\mu)$. Also, $f(\lambda + i\mu) = f(\lambda + \mu) \Longleftrightarrow \mu(\alpha - \beta)(1-i) = 0$. Since $\alpha \neq \beta$ according to Corollary~\ref{cor_quadractic_roots} and $\mu \neq 0$ as per the statement of Theorem~\ref{thm_main_theorem}, $\mu(\alpha - \beta)(1-i) = 0$ implies $i = 1$, which is not true. Thus, $f(\lambda + i\mu)$ and $f(\lambda + \mu)$ are distinct points. This implies that three non-collinear points on $\mathcal{C}'$ maps to three non-collinear points on the circle $|w| = \left|{(\lambda + \alpha)}/{\mu}\right|$. This completes the proof.  
\end{proof}

Please note that Theorem~\ref{thm_main_theorem} only states that the mapping of two nonconcentric circles in the $z$-plane results in concentric circles in the $w$-plane under \eqref{mobius_transformation}. However, it does not indicate anything how the \emph{three} regions: (i) enclosed between the two circles, (ii) enclosed within the inner circle, (iii) exterior to the outer circle in the $z$-plane are mapped in the $w$-plane. As discussed in Theorem~\ref{thm_mapping of regions between circles} below, the mapping of these regions depends on the roots $\alpha$ of \eqref{mobius_roots}. To facilitate further discussion, let $\alpha_{+}$ and $\alpha_{-}$ be the two roots of the quadratic equation \eqref{mobius_roots}, corresponding to the positive and negative signs of its discriminant. Define 
 \begin{equation}\label{alpha_s}
\alpha_s = \{\alpha_{\pm} \mid  |\alpha_s| = \min\{|\alpha_+|, |\alpha_{-}|\}\},
\end{equation}
and
\begin{equation}\label{alpha_l}
\alpha_{\ell} = \{\alpha_{\pm}  \mid |\alpha_{\ell}| = \max\{|\alpha_+|, |\alpha_{-}|\}\},
\end{equation}
as the magnitude-wise smaller and larger roots of \eqref{mobius_roots}, respectively. Based on this, we now state the following theorem:   

\begin{thm}[Mapping of enclosed regions]\label{thm_mapping of regions between circles}
Under the conditions specified in Theorem~\ref{thm_main_theorem}, the M\"{o}bius transformation \eqref{mobius_transformation} with $\alpha = \alpha_s$ (resp., $\alpha = \alpha_{\ell}$) preserves (resp., reverses) the interior-exterior mapping of the regions enclosed by the circles $\mathcal{C}$ and $\mathcal{C}'$ in the $w$-plane. Moreover, the mapping of the region enclosed between the two circles is always an annulus in the $w$-plane, irrespective of the roots $\alpha_s$ or $\alpha_{\ell}$.
\end{thm}

\begin{figure}[t]
	\centering{
	\subfigure[$\lambda > 0$]{\includegraphics[width=4.2cm]{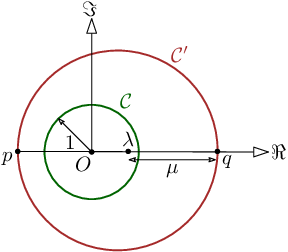}\label{A1}}\hspace*{0.2cm}
	\subfigure[$\lambda < 0$]{\includegraphics[width=4.2cm]{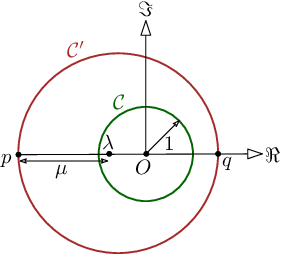}\label{A2}}
	}
	\caption{Case I: Circle $\mathcal{C}'$ encircles circle $\mathcal{C}$ and $\mu > 1 + |\lambda|$.}
	\label{case_1}
	\vspace*{-10pt}	
\end{figure}

\begin{figure}[t]
	\centering{
		\subfigure[$\lambda > 0$]{\includegraphics[width=4.2cm]{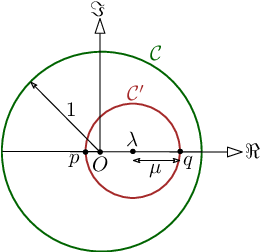}\label{A3}}\hspace*{0.2cm}
		\subfigure[$\lambda < 0$]{\includegraphics[width=4.2cm]{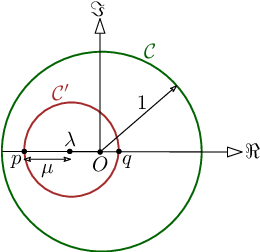}\label{A4}}
	}
	\caption{Case II: Circle $\mathcal{C}$ encircles circle $\mathcal{C}'$ and $\mu < 1 - |\lambda|$.}
	\label{case_2}
	\vspace*{-10pt}		
\end{figure}

\begin{figure*}[t]
	\centering{\hspace*{-0.5cm}
		\includegraphics[width=15.0cm]{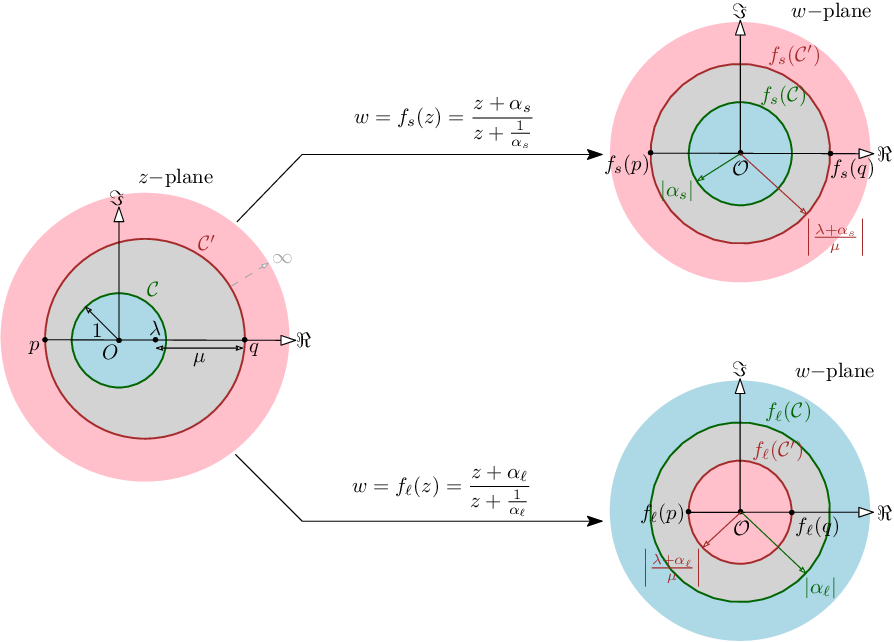}}
	\caption{Mapping of the circles $\mathcal{C}$, $\mathcal{C}'$ and their enclosed regions under M\"{o}bius transformation \eqref{mobius_transformation} with roots $\alpha_s$ and $\alpha_{\ell}$, defined by \eqref{alpha_s} and \eqref{alpha_l}, respectively.}
	\label{fig_mobius_mapping}
	\vspace*{-10pt}
\end{figure*}

\begin{proof}
Depending on the values of $\lambda$ and $\mu$, there arise two possible situations as shown in Fig.~\ref{case_1} and Fig.~\ref{case_2} where either the circle $\mathcal{C}'$ encircles circle $\mathcal{C}$ or otherwise. Note that Fig.~\ref{case_1} depicts a scenario where the center $\lambda$ of $\mathcal{C}'$ is enclosed by $\mathcal{C}$, however, the specific positioning of $\lambda$ does not matter provided it is known that $\lambda > 0$ or $\lambda < 0$. We address both these cases separately as follows:  
\begin{enumerate}[leftmargin=*]
		\item[Case I] Consider that the circle $\mathcal{C}'$ encircles circle $\mathcal{C}$ (Fig.~\ref{case_1}). In this situation, the hypotheses of Theorem~\ref{thm_main_theorem} are satisfied only if $\mu > 1 + |\lambda|$. The two circles touch each other at a common point if $\mu = 1 + |\lambda|$, and hence, excluded. Let $[p, q] = \mathcal{C}' \cap \{z \in \mathbb{C} \mid \Im(z) = 0\}$ be a diameter of $\mathcal{C}'$ along the real axis, where $p = \lambda - \mu$ and $q = \lambda + \mu$ denote the boundary points of the circle $\mathcal{C}'$ along its diameter (see Fig.~\ref{case_1}). Without loss of generality, we assume that $p+q > 0$ (see Fig.~\ref{A1}). If $p+q < 0$ (see Fig.~\ref{A2}), our proof is valid only with some minor modifications (as discussed subsequently). And if $p+q = 0$, we have nothing to prove (a meaningless case as the two circles are already concentric in the $z$-plane). Since the mapped circles $f(\mathcal{C})$ and $f(\mathcal{C}')$ are concentric in $w$-plane, the images of $z = p$ and $z = q$ must be symmetric with respect to $w = 0$ on the real axis in the $w$-plane. Thus, the M\"{o}bius transformation \eqref{mobius_transformation} must be such that $f(p) = -f(q)$, i.e., 
		\begin{equation}\label{eq_symmetry}
			\frac{p + \alpha}{p + \beta} + \frac{q + \alpha}{q + \beta} = 0,
		\end{equation}
		must have the real roots for $\alpha$. Rewriting \eqref{eq_symmetry}, we get
		\begin{equation}\label{mobius_roots_new}
			\alpha^2 + \frac{2(1+pq)}{p+q}\alpha + 1 = 0,
		\end{equation}
		which is essentially the condition \eqref{mobius_roots} by replacing $p = \lambda - \mu$ and $q = \lambda + \mu$. The discriminant $(1/4)$th of the above quadratic equation is
		\begin{equation*}
			\left[\frac{1+pq}{p+q}\right]^2 - 1 {=} \frac{1- p^2 - q^2 + p^2q^2}{(p+q)^2} {=} \frac{(1- p^2)(1 - q^2)}{(p+q)^2} > 0,
		\end{equation*}
		as $p < -1$ and $q > 1$ (see Fig.~\ref{A1}). Thus, the roots $\alpha$ of \eqref{eq_symmetry} are real. Moreover, both the roots are positive as their multiplication is $1$ and their product $-2(1+pq)/(p+q)$ is positive since $(1+pq)<0$ and $p+q > 0$. Also, none of the roots are unity, i.e., $\alpha \neq 1$ (see Corollary~\ref{cor_quadractic_roots}). Therefore, it can be concluded that one of the roots is between 0 and 1 and the other is greater than 1. On the other hand, if $p+q < 0$ (see Fig.~\ref{A2}), the sum of the roots is negative, and hence, both the roots will be negative with one smaller than $-1$ and other between $-1$ and $0$. Based on these roots, we next analyze the mapping of the three regions enclosed by the two circles $\mathcal{C}$ and $\mathcal{C}'$ under \eqref{mobius_transformation}. 
		\begin{itemize}[leftmargin=*]
			\item In case $p+q > 0$, $\lambda > 0$ and both roots of \eqref{mobius_roots_new} are positive (with one lying between $0$ and $1$, while other is greater than $1$). Assume that the radius $|(\lambda + \alpha)/\mu| = (\lambda + \alpha)/\mu$ of circle $f(\mathcal{C}')$ in the $w$-plane is larger than the radius $\alpha$ of circle $f(\mathcal{C})$ in the $w$-plane, that is, $(\lambda + \alpha)/\mu > \alpha$. Since $\mu > 1 + \lambda$ in this case (see Fig.~\ref{A1}), the preceding inequality implies that $(\lambda + \alpha) > \alpha(1 + \lambda) \implies \alpha < 1$. In other words, the smaller root $\alpha_s$ preservers the mapping of the circles. Further, the M\"{o}bius transform \eqref{mobius_transformation} has a pole only at $z = -\beta$, which is not in the connected set $\mathcal{R}$ comprising the circles $\mathcal{C}$ and $\mathcal{C}'$ and all the points between them. So, the mapping \eqref{mobius_transformation} associated with the root $\alpha_s$ (denote it by $f_s$ for notation simplicity) is continuous on $\mathcal{R}$. Now, it follows that the image, $f_s(\mathcal{R})$, is the connected set consisting of $f_s(\mathcal{C})$ and $f_s(\mathcal{C}')$ and all the points between them. Hence, the smaller root maintains the interior-exterior relationship of the regions enclosed by the two circles in the $w$-plane. The same reasoning holds for the mapping $f_{\ell}$ associated with larger root $\alpha_{\ell} > 1$, in which case, the interior-exterior relationship reverses. 
			\item If $p + q < 0$, $\lambda < 0$ and both roots of \eqref{mobius_roots_new} are negative (with one being smaller than $-1$ and other lies between $-1$ and $0$). Again, considering that the radius $|(\lambda + \alpha)/\mu| = -(\lambda + \alpha)/\mu$ is larger than the radius $-\alpha$, it is straightforward to see that $\alpha > -1$, following the similar steps as above. Alternatively, this implies that smaller root $\alpha_s$ (in absolute sense as in \eqref{alpha_s}) preservers the mapping of the circles $\mathcal{C}$ and $\mathcal{C}'$, and the regions enclosed by them in the $w$-plane, which is followed by the continuity argument as before.  
		\end{itemize}
		\item[Case II] Consider that the circle $\mathcal{C}$ encircles circle $\mathcal{C}'$ (Fig.~\ref{case_2}). In this situation, the hypotheses of Theorem~\ref{thm_main_theorem} are satisfied only if $\mu < 1 - |\lambda|$. Since $-1 < p < q < 1$ in this case, the roots of \eqref{mobius_roots_new} are again real. Further, these roots are negative (resp., positive) if $p + q > 0$ (resp., $p + q < 0$), unlike the previous case. Now, following the similar steps as above, one can draw the same conclusions in this case as well and are not discussed in detail for brevity. 
\end{enumerate}
	In summary, the region enclosed between $\mathcal{C}$ and $\mathcal{C}'$ always maps to an annulus between $f(\mathcal{C})$ and $f(\mathcal{C}')$ irrespective of the roots $\alpha$, while the mapping of inside and outside regions retains for $\alpha = \alpha_s$, and flips for $\alpha = \alpha_{\ell}$.
\end{proof}

Fig.~\ref{fig_mobius_mapping} illustrates Theorem~\ref{thm_mapping of regions between circles} (for the situation in Fig.~\ref{A1}) describing how the regions (enclosed by $\mathcal{C}$ in blue, between $\mathcal{C}$ and $\mathcal{C}'$ in gray, and outside $\mathcal{C}'$ in pink) in the $z$-plane maps to the $w$-plane with the choices of the roots $\alpha_s$ and $\alpha_{\ell}$. Please note that the outer circles in Fig.~\ref{fig_mobius_mapping} expand to infinity. One important observation from Theorem~\ref{thm_mapping of regions between circles} is that the condition $1 + \alpha \lambda > 0$ is always satisfied, which will be helpful in the subsequent analysis. 

\subsection{Illustrative Examples}

\begin{example}[$\mathcal{C}'$ encircles $\mathcal{C}$]\label{example_1}
For nonconcentric circles $\mathcal{C}: |z| = 1$ and $\mathcal{C}': |z - ({1}/{2})| = \sqrt{{5}/{2}}$, \eqref{mobius_roots} simplifies to $({1}/{2})\alpha^2 - (5/4)\alpha + ({1}/{2}) = 0$, since $\lambda = 1/2$ and $\mu = \sqrt{{5}/{2}}$. This equation yields two positive roots: $\alpha_{-} = 1/2$ and $\alpha_{+} = 2$. Selecting the smaller root $\alpha_s = 1/2$, the transformation \eqref{mobius_transformation}, given by $w= f_s(z) = (2z+1)/(2z+4)$, preserves the mapping, with $\mathcal{C}$ mapping to $f_s(\mathcal{C}): |w| = {1}/{2}$ and $\mathcal{C}'$ to $f_s(\mathcal{C}'): |w| = \sqrt{2/5}$. Conversely, for the larger root $\alpha_{\ell} = 2$, the transformation $w= f_{\ell}(z) = (2z+4)/(2z+1)$ reverses the mapping, with $\mathcal{C}$ mapping to $|w| = 2$ and $\mathcal{C}'$ to $|w| = \sqrt{{5}/{2}}$. Consequently, the interior-exterior relationship of the enclosed regions is preserved for $\alpha_s$ and reversed for $\alpha_{\ell}$ in the $w$-plane, as per Theorem~\ref{thm_mapping of regions between circles}.
\end{example}

\begin{example}[$\mathcal{C}$ encircles $\mathcal{C}'$]\label{example_2}
For nonconcentric circles $\mathcal{C}: |z| = 1$ and $\mathcal{C}': |z - ({2}/{5})| = {2}/{5}$, \eqref{mobius_roots} leads to $\alpha^2 + ({5}/{2})\alpha + 1 = 0$, since $\lambda = \mu = {2}/{5}$. This equation yields both negative roots: $\alpha_{+} = -1/2$ and $\alpha_{-} = -2$. Selecting $\alpha_s = \alpha_{+} = -1/2$, the transformation $w= f_s(z) = (2z - 1)/(2z - 4)$ preserves the mapping, with $\mathcal{C}$ mapping to $f_s(\mathcal{C}): |w| = {1}/{2}$ and $\mathcal{C}'$ mapping to $f_s(\mathcal{C}'): |w| = {1}/{4}$. Conversely, for the other root $\alpha_{\ell} = -2$, the transformation $w= f_{\ell}(z) = (2z-4)/(2z-1)$ reverses the mapping, with $\mathcal{C}$ mapping to $f_{\ell}(\mathcal{C}): |w| = 2$ and $\mathcal{C}'$ mapping to $f_{\ell}(\mathcal{C}'): |w| = 4$. Again, the enclosed regions are mapped accordingly, as per Theorem~\ref{thm_mapping of regions between circles}. 
\end{example}

\begin{remark}[Coordinate transformation]\label{remark_coordinate_transformation}
Suppose the circles are not in the standard form as in Theorem~\ref{thm_main_theorem}, the following remarks apply: 
\begin{itemize}[leftmargin=*]
\item If the circles are given by $|z| = \varrho$ and $|z-\lambda| = \mu$, we can transform them into the standard form by substituting $z = \varrho\tilde{z}$, leading to $|\tilde{z}| = 1$ and $|\tilde{z} - \tilde{\lambda}| = \tilde{\mu}$, where $\tilde{\lambda} = \lambda/\varrho$ and $\tilde{\mu} = \mu/\varrho$. After obtaining the solution, we revert to the original plane using the relation $\tilde{z} = {z}/{\varrho}$.

 \item When neither of the given circles is centered at the origin, a preliminary step involves translating the center of one circle to the origin by a linear mapping $z = \hat{z} + c_0$, where $c_0$ is a constant. For instance, to map the circles $|z-\lambda_1| = \mu_1$ and $|z-\lambda_2| = \mu_2$ to circles centered at $w = 0$, let $z = \hat{z} + \lambda_1$, yielding $|\hat{z}| = \mu_1$ and $|\hat{z} - (\lambda_2 - \lambda_1)| = \mu_2$. Then, using $\hat{z} = \mu_1\tilde{z}$, the circles are brought to the standard form $|\tilde{z}| = 1$ and $|\tilde{z} - ((\lambda_2 - \lambda_1)/\mu_1)| = \mu_2/\mu_1$.
 
 \item If $\lambda$ is not real, implying that the center of circle $\mathcal{C}'$ in Theorem~\ref{thm_main_theorem} lies on the imaginary axis. A preliminary rotation of the complex plane takes care of this: let $\varphi = \arg{(\lambda)}$ and $\tilde{z} = z{\rm e}^{-i\varphi}$. This transforms the original circle to $|\tilde{z}{\rm e}^{i\varphi} - |\lambda|{\rm e}^{i\varphi}| \implies |\tilde{z} - \tilde{\lambda}| = \mu$, where $\tilde{\lambda} = |\lambda|$ is a positive real number.
\end{itemize}
\end{remark}

\section{System Model and Problem Description} \label{section_3}

\subsection{System Model}
We consider an under-actuated unicycle robot that moves in the $\mathbb{R}^2$ space with constant forward linear speed $v \in \mathbb{R}_{+}$ and variable angular speed $\omega \in \mathbb{R}$. Let $r = [x, y]^\top \in \mathbb{R}^2$ be the position of the robot with its velocity vector heading in the direction $\theta \in \mathbb{S}^1$. For ease of analysis, we associate this $\mathbb {R}^2$ plane with a complex plane $\mathbb {C}$ through the mapping $[\mathfrak{p}, \mathfrak{q}]^\top \mapsto \mathfrak{p} + i\mathfrak{q}$, and represent the position and velocity of the robot as
\begin{equation}\label{position_actual_plane}
	r = x + iy = |r|{\rm e}^{i\phi},
\end{equation}
and $\dot{r} = v{\rm e}^{i\theta} = v(\cos\theta + i\sin\theta) \in \mathbb{C}$, respectively. Here, both the position angle $\phi$ and the heading angle $\theta$ are measured in the anticlockwise direction from the positive real axis. With this, the robot kinematics is described by
\begin{equation}\label{dyn}
	\dot{r}  = v{\rm e}^{i\theta}, \qquad \dot{\theta}  = \omega,
\end{equation}
where $\omega$ is the turn-rate controller to be designed for meeting the desired control objective(s). If $\omega \equiv 0$, the robot moves in the straight line with slope $\theta(0)$, while if $\omega \neq 0$ is constant, the robot moves on a circular orbit of the radius $v/|\omega| \in \mathbb{R}_+$. For characterizing the motion, we follow the convention that, if $\omega > 0$, the robot moves in the anticlockwise direction, and if $\omega < 0$, the robot moves in the clockwise direction.  

\begin{figure}[t!]
	\centering
	{\includegraphics[width=8.0cm]{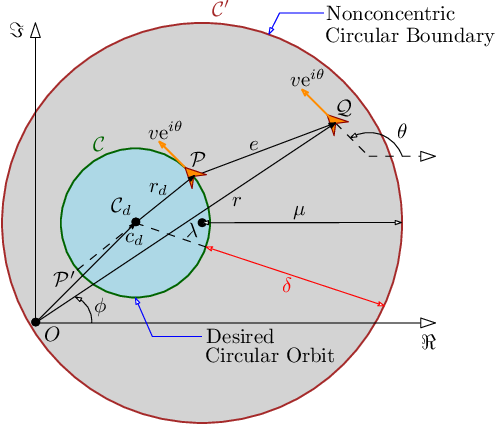}}
	\caption{Unicycle robot tracking desired circular orbit $\mathcal{C}$ with its motion bounded within nonconcentric circular boundary $\mathcal{C}'$.}
	\label{fig_problem}	
	\vspace*{-10pt}
\end{figure}

\subsection{Problem Description}
  As illustrated in Fig.~\ref{fig_problem}, our main objective is to stabilize the robot's motion about the desired circular orbit $\mathcal{C}$ while assuring that its trajectories remain bounded within a nonconcentric circular boundary $\mathcal{C}'$ throughout its motion. By saying the desired circle, we mean that the circle has the desired radius $r_d \in \mathbb{R}_{+}$ and center $c_d = c_{dx} + ic_{dy} \in \mathbb{C}$. For the robot to move around $\mathcal{C}$, the absolute value of the error $e \in \mathbb{C}$ between the current position of the robot (point $\mathcal{Q}$) and the desired location (point $\mathcal{P}$) on the circular orbit must be minimized. The choice of the point $\mathcal{P}$ is followed by the fact that the velocity vector $\dot{r}$ must be tangent to $\mathcal{C}$ at $\mathcal{P}$ if the robot moves on the circle $\mathcal{C}$ at that instant of time $t$. It is easy to observe that other such point on $\mathcal{C}$ is $\mathcal{P}'$ which is diametrically opposite to $\mathcal{P}$. Either of these points can be considered, however, the point $\mathcal{P}$ (resp., $\mathcal{P}'$) is suitable if the robot wishes to move in the anticlockwise (resp., clockwise) direction. By rotating the unit vector ${\rm e}^{i\theta}$ by $\pi/2$ radians in the clockwise (resp., anticlockwise) direction, we obtain $-i {\rm e}^{i \theta}$ (resp., $i {\rm e}^{i \theta}$), representing the unit vector along $\mathcal{C}_d\mathcal{P}$ (resp., $\mathcal{C}_d\mathcal{P}')$. Using the vector addition law, it can be written that $\mathcal{PQ} = \mathcal{C}_d\mathcal{Q} - \mathcal{C}_d\mathcal{P} = (O\mathcal{Q} - O\mathcal{C}_d) - \mathcal{C}_d\mathcal{P}$ and $\mathcal{P}'\mathcal{Q} = \mathcal{C}_d\mathcal{Q} - \mathcal{C}_d\mathcal{P}' = (O\mathcal{Q} - O\mathcal{C}_d) - \mathcal{C}_d\mathcal{P}'$, which allows us to express the error vector as $e = (r - c_d) \pm i r_d {\rm e}^{i\theta}$, where `$+$' (resp., `$-$') sign corresponds to the motion of the robot in the anticlockwise (resp., clockwise) direction. 
 
Without loss of generality, we describe the two nonconcentric circles in Fig.~\ref{fig_problem} in the framework of Theorem~\ref{thm_main_theorem} (via a suitable coordinate transformation) and consider that $r_d = 1$ unit and $c_d = 0 + i0$, that is, the desired circular orbit is $\mathcal{C}: |z| = 1$ and assume that the outer circular boundary has the form $\mathcal{C}': |z - \lambda| = \mu$, with center $\lambda \in \mathbb{R}\setminus\{0\}$ and radius $\mu \in \mathbb{R}_{+}$. Here, since $\mathcal{C}'$ encircles $\mathcal{C}$, it must hold that $\mu > 1 + |\lambda|$, in accordance with Theorem~\ref{thm_mapping of regions between circles}. With these considerations, the preceding error vector can be expressed as:
\begin{equation}\label{error}
 e = r  \pm i {\rm e}^{i\theta}.
\end{equation}     
In light of the above discussion, we now formally state the problem considered in this paper:

\begin{problem}[Problem in actual plane]\label{problem}
Consider the circles $\mathcal{C}: |z| = 1$ and $\mathcal{C}': |z - \lambda| = \mu$, where $\lambda$ and $\mu$ are as defined in Theorem~\ref{thm_main_theorem} and satisfy the condition $\mu > 1 + |\lambda|$ (i.e., $\mathcal{C}'$ encircles $\mathcal{C}$ in accordance with Fig.~\ref{fig_problem}). Design the control law $\omega$ such that the unicycle robot \eqref{dyn} with initial position $r(0) \in  \mathcal{Z} \coloneqq \{z_r \in \mathbb{C} \mid |z_r - \lambda| < \mu\}$ asymptotically stabilizes on the desired circle $\mathcal{C}$ and its trajectory remains bounded within nonconcentric circular boundary $\mathcal{C}'$, that is, $e(t) \to 0$ (equivalently, $|r(t)| \to 1$), as $t \to \infty$ and $r(t) \in \mathcal{Z}$ for all $t \geq 0$. 
\end{problem}

According to Problem~\ref{problem}, the robot is not allowed to move exterior to $\mathcal{C}'$, however, it can move inside $\mathcal{C}$ and the region enclosed within the two circles. As shown in Fig.~\ref{fig_problem}, the radial distance $\delta \in \mathbb{R}_{+}$ between the two circles is not uniform and varies in space as we traverse around $\mathcal{C}$, unlike \cite{jain2019trajectory}. Our solution approach to Problem~\ref{problem} relies on transforming the motion of the robot in the actual plane to the transformed plane using M\"{o}bius transformation \eqref{mobius_transformation}. This essentially transforms Problem~\ref{problem} to the problem of motion stabilization within the concentric circles in the transformed plane. We emphasize that this is not just a space transformation problem, instead, the motion transformation problem where it becomes important how the velocity $\dot{r}$ and control $\omega$ of the unicycle robot \eqref{dyn} transforms and maps back to the actual plane, in addition to its position under transformation \eqref{mobius_transformation}. We try to answer these questions in the upcoming sections.

\section{Problem Formulation and Control Design in the transformed plane}\label{section_4}
In this section, we begin by restating the problem and designing the control in the transformed plane, following Theorem~\ref{thm_main_theorem} and Theorem~\ref{thm_mapping of regions between circles} under the M\"{o}bius transformation \eqref{mobius_transformation}. This control design needs to be mapped back to the actual plane to obtain the real control input $\omega$. We proceed in this direction as follows:  

\subsection{System Model and Problem Formulation in Transformed Plane}
Under \eqref{mobius_transformation}, the position \eqref{position_actual_plane} of the robot is transformed from the actual plane to the transformed plane as 
\begin{equation}\label{position_trnsformed_plane}
	\rho = f(r) = \frac{r+\alpha}{r+\beta} = \frac{\alpha(r+\alpha)}{1+\alpha r} \coloneqq |\rho|{\rm e}^{i\psi},
\end{equation}
where $|\rho|$ is the magnitude and $\psi = \arg(\rho)$ is the argument of $\rho = \rho_x + i\rho_y  \in \mathbb{C}$, measured in the anticlockwise-direction from the positive real-axis in the transformed plane. Similar to \eqref{dyn}, the equation of motion in the transformed plane can be expressed as 
\begin{equation}\label{tdyn}
	\dot{\rho}  = |\dot{\rho}|{\rm e}^{i\gamma}, \qquad \dot{\gamma}  = \Omega,
\end{equation}
where $|\dot{\rho}|$ is the linear speed and $\gamma \in \mathbb{S}^1$ is the heading angle of the unicycle robot in the transformed plane. Unlike the constant speed $v$ in the actual plane, the speed $|\dot{\rho}|$ in the transformed plane is no longer constant and varies with time. The angular rate $\dot{\gamma}$ is controlled using the (virtual) control law $\Omega$ in the transformed plane. However, the sense of rotation of the robot's velocity vector in the transformed plane is governed by the roots of \eqref{mobius_roots}. This is discussed in the following theorem: 

\begin{thm}[Sense of rotation of velocity vectors in two planes]\label{thm_sense_of_rotation}
Consider the robot models \eqref{dyn} and \eqref{tdyn} in the actual and transformed planes, respectively. Under the M\"{o}bius transformation \eqref{mobius_transformation}, the sense of rotation of the robot's velocity vector in the two planes is the same if $\alpha = \alpha_s$, and it is opposite if $\alpha = \alpha_{\ell}$.   
\end{thm}

\begin{proof}
Please refer to Appendix~\ref{appendix_A}. 
\end{proof}

Let us now turn our focus on reformulating Problem~\ref{problem} in the transformed plane. Since the desired circular orbit $\mathcal{C}$ maps to the circle $f(\mathcal{C})$ with radius $|\alpha|$ in the transformed plane (refer to Fig.~\ref{fig_mobius_mapping}), the positional error in the transformed plane can be written analogously to \eqref{error} as
\begin{equation}\label{error_transformed_plane}
	\mathcal{E} = \rho \pm i |\alpha|{\rm e}^{i\gamma} = \rho + i \sigma {\rm e}^{i\gamma},
\end{equation}
where, 
\begin{equation}\label{sigma}
	\sigma \coloneqq \pm|\alpha| \neq 0,	
\end{equation}
and the `$+$' and `$-$' signs are used in the same spirit as in \eqref{error}. Following Theorem~\ref{thm_sense_of_rotation}, this corresponds to the fact that $\sigma = +|\alpha|$ if $\alpha = \alpha_s$, and $\sigma = -|\alpha|$ if $\alpha = \alpha_{\ell}$. The equivalent control objective in the transformed plane is to minimize $\mathcal{E}$ in \eqref{error_transformed_plane}. Depending on the root $\alpha_s$ or $\alpha_{\ell}$ of \eqref{mobius_roots}, Problem~\ref{problem} can be inferred either as a \emph{trajectory-constraining problem} or an \emph{obstacle-avoidance problem} in the transformed plane, respectively. This is illustrated in Fig.~\ref{fig_problem_transformed_plane} where the radial distance $\delta_T$ between the two circular boundaries is uniform and is given by
\begin{equation}\label{delta_transformed}
	\delta_T = 
	\begin{cases}
		\left|\frac{\lambda + \alpha}{\mu}\right| - |\alpha|,  & \text{if } \alpha = \alpha_s\\
		|\alpha| - \left|\frac{\lambda + \alpha}{\mu}\right|,  & \text{if } \alpha = \alpha_{\ell}.
	\end{cases}	
\end{equation} 

It is important to note that, analogous to actual error \eqref{error} in Fig.~\ref{fig_problem}, the transformed error \eqref{error_transformed_plane} is defined with respect to mapping $f(\mathcal{C})$ of the desired circular orbit $\mathcal{C}$ for both $\alpha_s$ and $\alpha_{\ell}$ (see Fig.~\ref{fig_problem_transformed_plane}). The problem can be formally restated in the transformed plane as:

\begin{figure}[t]
	\centering{
		\subfigure[Trajectory-constraining]{\includegraphics[width=4.4cm]{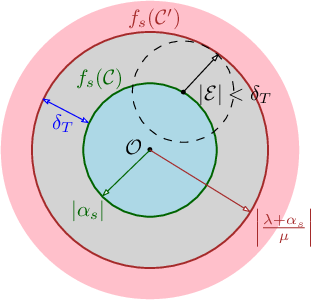}\label{trajectory_constraining}}\hspace*{0.2cm}
		\subfigure[Obstacle-avoidance]{\includegraphics[width=4.4cm]{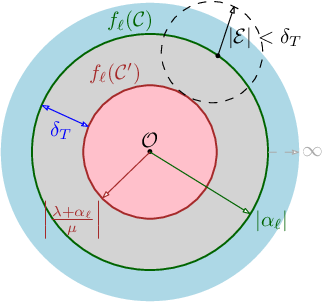}\label{collision_avoidance}}
	}
	\caption{Problem in the transformed plane with choices of roots $\alpha_s$ and $\alpha_{\ell}$.}
	\label{fig_problem_transformed_plane}
	\vspace*{-10pt}	
\end{figure}

\begin{problem}[Problem in transformed plane]\label{problem_transformed}
Consider the concentric circles $f(\mathcal{C}): |w| = |\alpha|$ and  $f(\mathcal{C}'): |w| = \left| \frac{\lambda + \alpha}{\mu}\right|$, obtained under the  M\"{o}bius transformation \eqref{mobius_transformation}, where $\lambda$ and $\mu$ are defined in Theorem~\ref{thm_main_theorem} and $\alpha$ is the root of \eqref{mobius_roots}. Based on the choice of $\alpha$ as discussed in Theorem~\ref{thm_mapping of regions between circles}, Problem~\ref{problem} is equivalent to the following two problems in the transformed plane:
	\begin{enumerate}[leftmargin=*]
		\item[\emph{(P1)}] Trajectory-Constraining Problem (see Fig.~\ref{trajectory_constraining}): If $\alpha = \alpha_s$. Design the control $\Omega$ such that the robot model \eqref{tdyn} with initial position $\rho(0) \in  \mathcal{W}_s \coloneqq \{w_{\rho} \in \mathbb{C} \mid |w_{\rho}| < \left| \frac{\lambda + \alpha_s}{\mu}\right|\}$ asymptotically stabilizes on the circle $f_s(\mathcal{C})$ and its trajectory remains bounded within the concentric circle $f_s(\mathcal{C}')$ during stabilization, i.e., $\mathcal{E}(t) \to 0$ (equivalently, $|\rho(t)| \to |\alpha_s|$) as $t \to \infty$ and $\rho(t) \in  \mathcal{W}_s$ for all $t \geq 0$. 
		\item[\emph{(P2)}] Obstacle-Avoidance Problem (see Fig.~\ref{collision_avoidance}): If $\alpha = \alpha_{\ell}$. Design the control $\Omega$ such that the robot model \eqref{tdyn} with initial position $\rho(0) \in  \mathcal{W}_{\ell} \coloneqq \{w_{\rho} \in \mathbb{C} \mid |w_{\rho}| > \left| \frac{\lambda + \alpha_{\ell}}{\mu}\right|\}$ asymptotically stabilizes on the circle $f_{\ell}(\mathcal{C})$ and its trajectory remains exterior to the concentric circle $f_{\ell}(\mathcal{C}')$ during stabilization, i.e., $\mathcal{E}(t) \to 0$ (equivalently, $|\rho(t)| \to |\alpha_{\ell}|$) as $t \to \infty$ and $\rho(t) \in  \mathcal{W}_{\ell}$ for all $t \geq 0$. 
	\end{enumerate}
\end{problem}

To solve Problem~\ref{problem_transformed}, we express the requirements on $\rho$ in terms of the transformed error $\mathcal{E}$ to be used in the logarithmic BLF. In this context, the condition $|\mathcal{E}| < \delta_T$, which essentially represents a disk of radius $\delta_T$ as shown in Fig.~\ref{fig_problem_transformed_plane}, becomes crucial in both the cases of $\alpha$, which is the topic of the next section. 

\subsection{Control Design in Transformed Plane}
 In this subsection, we first design the controller $\Omega$ to solve the Problem~\ref{problem_transformed} in the transformed plane and show how this renders a solution to the original Problem~\ref{problem}. A connection between $\Omega$ and the actual control law $\omega$ is established in Section~\ref{section_5}. 
 
Concerning the robot motion in the transformed plane, we consider the BLF:
\begin{equation}\label{Lyapunov_function}
	\mathcal{S}(\mathcal{E}) = \mathcal{S}(\rho, \gamma) \coloneqq \frac{1}{2}  \ln \left(\frac{\delta_T^2}{\delta_T^2 - |\mathcal{E}|^2}\right),
\end{equation}
where, `$\ln$' denotes the natural logarithm and the constant $\delta_T > 0$ is defined in \eqref{delta_transformed}. The Lyapunov function $\mathcal{S}(\rho, \gamma)$ exhibits the following properties: (i) it is positive definite and continuously differentiable in the region $|\mathcal{E}| < \delta_T$ \cite{tee2009barrier}, and (ii) $\mathcal{S}(\rho, \gamma) = 0$ whenever $\mathcal{E} = 0$. It means that the minimization of $\mathcal{S}(\mathcal{E})$ corresponds to the situation when the (virtual) robot moves on the circle $f_s(\mathcal{C})$ (resp., $f_{\ell}(\mathcal{C})$) for $\alpha_s$ (resp., $\alpha_{\ell}$) (see Fig.~\ref{fig_problem_transformed_plane}). The time-derivative of $\mathcal{S}(\mathcal{E})$, along the dynamics \eqref{tdyn}, is obtained as
\begin{equation}\label{Lyapunov_function_derivative}
	\dot{\mathcal{S}} = \frac{1}{2} \left[\frac{\delta_T^2 - |\mathcal{E}|^2}{\delta_T^2}\right]\times\left[\frac{\delta_T^2\frac{d}{dt}{|\mathcal{E}|^2}}{(\delta_T^2 - |\mathcal{E}|^2)^2}\right] =  \frac{1}{2}\left(\frac{\frac{d}{dt}|\mathcal{E}|^2}{\delta_T^2 - |\mathcal{E}|^2}\right),	
\end{equation}
where $\frac{1}{2}\frac{d}{dt}|\mathcal{E}|^2 = \langle \mathcal{E}, \dot{\mathcal{E}} \rangle $. From \eqref{error_transformed_plane}, $\dot{\mathcal{E}}$ is obtained as:
\begin{equation}\label{error_transformed_plane_dot}
	\dot{\mathcal{E}} =	\dot{\rho} - \sigma {\rm e}^{i\gamma} \dot{\gamma} = |\dot{\rho}|{\rm e}^{i\gamma} - \sigma {\rm e}^{i\gamma} \dot{\gamma} = (|\dot{\rho}| - \sigma\Omega){\rm e}^{i\gamma},
\end{equation}
using \eqref{tdyn}, where $(|\dot{\rho}| - \sigma\dot{\gamma}) \in \mathbb{R}$. Substituting for $\mathcal{E}$ and $\dot{\mathcal{E}}$ from \eqref{error_transformed_plane} and \eqref{error_transformed_plane_dot}, respectively, we have $\frac{1}{2}\frac{d}{dt}|\mathcal{E}|^2 = \langle \rho + i\sigma{\rm e}^{i\gamma},  {\rm e}^{i\gamma} \rangle (|\dot{\rho}| - \sigma\Omega) = [\langle \rho,  {\rm e}^{i\gamma} \rangle + \sigma \langle i {\rm e}^{i\gamma},  {\rm e}^{i\gamma} \rangle](|\dot{\rho}| - \sigma \Omega) = \langle \rho,  {\rm e}^{i\gamma} \rangle (|\dot{\rho}| - \sigma\Omega)$, as $\langle i {\rm e}^{i\gamma},  {\rm e}^{i\gamma} \rangle = 0$. Thus, \eqref{Lyapunov_function_derivative} simplifies as
\begin{equation}
	\dot{\mathcal{S}} = \frac{\langle \rho, {\rm e}^{i\gamma} \rangle }{\delta_T^2 - |\mathcal{E}|^2}(|\dot{\rho}|- \sigma \Omega).
\end{equation} 

Based on the potential \eqref{Lyapunov_function}, we propose the control $\Omega$ in the following theorem. 

\begin{thm}[Control law $\Omega$]\label{thm_control_law_transformed}
Consider the transformed robot model \eqref{tdyn} and assume that the initial states $[\rho(0), \gamma(0)]^\top$ belong to the set $\mathcal{W}_{\delta_T} \coloneqq \{[\rho, \gamma]^\top \in \mathbb{C} \times \mathbb{S}^1 \mid |\mathcal{E}| < \delta_T\}$, where $\mathcal{E}$ and $\delta_T$ are defined in \eqref{error_transformed_plane} and \eqref{delta_transformed}, respectively. Let \eqref{tdyn} be governed by the control law
	\begin{equation}\label{control_law_transformed}
		\Omega = \frac{1}{\sigma} \left[|\dot{\rho}|+ \kappa \frac{\langle \rho, {\rm e}^{i\gamma} \rangle }{\delta_T^2 - |\mathcal{E}|^2}\right],
	\end{equation}
	where $\kappa > 0$ is the control gain and $\sigma$ is given by \eqref{sigma}. Then, the following properties hold:
	\begin{enumerate}[leftmargin=*]
		\item[(a)] The (virtual) robot asymptotically stabilizes to the desired circle $f(\mathcal{C})$ in the set $\mathcal{W}_{\delta_T}$ with its direction of motion determined by the roots $\alpha_s$ and $\alpha_{\ell}$, as per Theorem~\ref{thm_sense_of_rotation}. 
		\item[(b)] Depending on $\alpha_s$ and $\alpha_{\ell}$, the trajectories of \eqref{tdyn} remain bounded in the regions described by: 
		\begin{subequations}\label{rho_bounds}
		\begin{align}
			\label{rho_bound_alpha_s}2|\alpha| - \left|\frac{\lambda + \alpha}{\mu}\right| < |\rho(t)| < \left|\frac{\lambda + \alpha}{\mu}\right|, \quad & \text{if } \alpha = \alpha_s\\	
			\label{rho_bound_alpha_l}\left|\frac{\lambda + \alpha}{\mu}\right| < |\rho(t)| < 2|\alpha| - \left|\frac{\lambda + \alpha}{\mu}\right|, \quad & \text{if } \alpha = \alpha_{\ell},
		\end{align} 
	\end{subequations}
		for all $t \geq 0$.
	\end{enumerate}
\end{thm}

\begin{proof}
We prove the statement separately as follows:
\begin{enumerate}[leftmargin=*]
\item [(a)]Consider the potential function \eqref{Lyapunov_function} whose time-derivative, along the dynamics \eqref{tdyn}, is given by \eqref{Lyapunov_function_derivative}. Choosing $\Omega$ as given in \eqref{control_law_transformed}, $\dot{\mathcal{S}}$ becomes
\begin{equation}\label{Lyapunov_function_derivative_new}
\dot{\mathcal{S}} = -\kappa \left[\frac{\langle \rho, {\rm e}^{i\gamma} \rangle }{\delta_T^2 - |\mathcal{E}|^2} \right]^2 \leq 0,
\end{equation}
for $\kappa > 0$. This implies that $\mathcal{S}(\rho, \gamma) \leq \mathcal{S}(\rho(0), \gamma(0))$, $\forall t \geq 0$, along the trajectories of \eqref{tdyn}. Let us consider the set $\Xi \coloneqq \left\{[\rho, \gamma]^\top \in \mathcal {W}_{\delta_T} \mid \mathcal{S}(\rho, \gamma) \leq \mathcal{S}(\rho(0), \gamma(0))\right\}$. Note that the set $\Xi$ is compact and positively invariant because $\mathcal{S}(\rho, \gamma)$ is positive definite and continuously differentiable in $\mathcal{W}_{\delta_T}$, and $\dot{\mathcal{S}} \leq 0$ along the solutions of \eqref{tdyn}. According to Lasalle's invariance principle \cite{khalil2002nonlinear}, all the solutions of \eqref{tdyn}, under the control \eqref{control_law_transformed}, converge to the largest invariant set $\Gamma \subset \Xi$, where $\dot{\mathcal{S}} = 0$, i.e., $\Gamma = \{[\rho, \gamma]^\top \in \mathcal{W}_{\delta_T} \mid \dot{\mathcal{S}} = 0\}$. From \eqref{Lyapunov_function_derivative_new}, it is clear that $\dot{\mathcal{S}} = 0$ if 
\begin{equation}\label{largest_invariant_set}
\frac{\langle \rho, {\rm e}^{i\gamma} \rangle}{\delta_T^2 - |\mathcal{E}|^2} = 0 \implies \langle \rho, {\rm e}^{i\gamma} \rangle = 0,
\end{equation}
in $\Gamma$. Substituting \eqref{largest_invariant_set} into \eqref{control_law_transformed} implies that $\Omega = |\dot{\rho}|/\sigma$ in the set $\Gamma$. Consequently, it follows from \eqref{error_transformed_plane_dot} that $\dot{\mathcal{E}} = 0$ in the set $\Gamma$. Now, substituting for $\rho$ from \eqref{position_trnsformed_plane} into \eqref{largest_invariant_set}, results in $\langle \mathcal{E} - i \sigma {\rm e}^{i\gamma}, {\rm e}^{i\gamma} \rangle = 0 \implies \langle \mathcal{E},  {\rm e}^{i\gamma} \rangle - \sigma \langle i{\rm e}^{i\gamma}, {\rm e}^{i\gamma} \rangle = 0 \implies \langle \mathcal{E},  {\rm e}^{i\gamma} \rangle = 0$, since $\langle i{\rm e}^{i\gamma}, {\rm e}^{i\gamma} \rangle = 0$. For $\langle \mathcal{E},  {\rm e}^{i\gamma} \rangle = 0$ to hold, it is needed that $\mathcal{E} = 0$, as $\dot{\mathcal{E}} = 0$ in $\Gamma$. This, together with the fact that $\Omega = |\dot{\rho}|/\sigma$ in $\Gamma$, implies that the trajectories of \eqref{tdyn} asymptotically converge around the circle of radius $|\sigma| = |\alpha|$ with the linear speed $|\dot{\rho}|$ and the direction of rotation determined by $\sigma$ (see Theorem~\ref{thm_sense_of_rotation}). Alternatively, $\mathcal{E} \to 0$ in the set $\Gamma \subset \Xi \subset \mathcal{W}_{\delta_T}$, as $t \to \infty$.     
\item [(b)] Since $\mathcal{S}(\mathcal{E}) \to \infty$ as $|\mathcal{E}| \to \delta_T$ and $\dot{\mathcal{S}} \leq 0$, it directly follows from Lemma~\ref{lem_BLF} that $|\mathcal{E}(t)| < \delta_T$ for all $t \geq 0$. Substituting $\mathcal{E}(t)$ from \eqref{error_transformed_plane} into the preceding inequality, we get $|\rho(t) + i\sigma{\rm e}^{i\gamma(t)}| < \delta_T$ for all $t \geq 0$. Using the triangle inequality $\Big{|}|w_1| - |w_2|\Big{|} \leq |w_1 + w_2|$ for $w_1, w_2 \in \mathbb{C}$, we can write  	
\begin{equation*}
\Big{|} |\rho(t)| - |\sigma| \Big{|} \leq \Big{|} \rho(t) + i\sigma{\rm e}^{i\gamma(t)} \Big{|}  < \delta_T, \quad \forall t\geq 0,
\end{equation*}
considering $w_1 = \rho(t)$ and $w_2 = i\sigma{\rm e}^{i\gamma(t)}$. This implies that $\pm (|\rho(t)| - |\alpha|)   < \delta_T, \forall t \geq 0$, where we have substituted $|\sigma| = |\alpha|$ from \eqref{sigma}. Upon taking the `$+$' sign in this inequality, we get $|\rho(t)| < |\alpha| + \delta_T, \forall t \geq 0$ and taking `$-$' sign, we get $|\alpha| - \delta_T < |\rho (t)|, \forall t \geq 0$. Consequently, $|\alpha| - \delta_T < |\rho(t)| < |\alpha| + \delta_T$ for all $t \geq 0$. Now, substituting for $\delta_T$ from \eqref{delta_transformed}, we get the required result stated in \eqref{rho_bounds}. 
\end{enumerate}
Since it has been established that $|\mathcal{E}(t)| < \delta_T$ for all $t \geq 0$, the control \eqref{control_law_transformed} always remains finite for any solution trajectory. This completes the proof. 
\end{proof}

Our next step is to show how Theorem~\ref{thm_control_law_transformed} solves the problem in the actual plane, that is, Problem~\ref{problem}. Although \eqref{rho_bounds} provides both upper and lower bounds on $|\rho(t)|$, only the upper bound in \eqref{rho_bound_alpha_s} for $\alpha = \alpha_s$, and the lower bound in \eqref{rho_bound_alpha_l} for $\alpha = \alpha_{\ell}$ are important as long as the problems (P1) and (P2) are concerned. Considering only these necessary bounds, we obtain bounds on the position $|r(t)|$ in the actual plane in the below theorem. These can be similarly obtained for the other bounds in \eqref{rho_bounds} and omitted for brevity.       

\begin{thm}[Solution to Problem~\ref{problem}]\label{thm_problem_1_solution}
Under the conditions given in Theorem~\ref{thm_control_law_transformed}, Problem~\ref{problem} is solved in the following sense:
\begin{enumerate}[leftmargin=*]
	\item[(a)] In steady-state, robot \eqref{dyn} in the actual plane converges to the desired circular orbit, i.e., $|r(t)| = 1$ as $t \to \infty$. 
	\item[(b)] The robot's position $r(t)$ in the actual plane remains in the set $|r(t) - \lambda| < \mu$ for all $t \geq 0$, where $\lambda$ and $\mu$ are as defined in Theorem~\ref{thm_main_theorem}. 
\end{enumerate}
\end{thm}

\begin{proof}
The statements are proved separately as follows:
\begin{enumerate}[leftmargin=*]
\item[(a)] In steady-state, since trajectories of \eqref{tdyn} converge to the circle $f(\mathcal{C}): |w| = |\alpha|$ under the control law \eqref{control_law_transformed}, it holds that $|\rho| = |\alpha|$ at $t \to \infty$. Consequently, it follows from \eqref{position_trnsformed_plane}, after taking modulus on both sides and substituting $|\rho| = |\alpha|$, that $|r + \alpha| = |1 + \alpha r|$ as $t \to \infty$. Squaring both sides yields $|r + \alpha|^2 = |1 + \alpha r|^2 \implies (r + \alpha)(r^\ast + \alpha) = (1 + \alpha r)(1 + \alpha r^\ast) \implies |r|^2 + \alpha(r + r^\ast) + \alpha^2 = 1 + \alpha^2|r|^2 + \alpha(r + r^\ast) \implies (1 - \alpha^2)|r|^2 = (1 - \alpha^2) \implies |r|^2 = 1$, since $\alpha \neq \pm 1$ (see Corollary~\ref{cor_quadractic_roots}). This concludes that the robot \eqref{dyn} moves around the desired circular orbit $\mathcal{C}$ in the actual plane in the steady state. 
\item[(b)] Denote $\nu = |(\lambda + \alpha)/{\mu}|$. We now consider the following two cases depending on $\alpha_s$ and $\alpha_{\ell}$.
\begin{enumerate}[leftmargin=*]
\item [Case I] If $\alpha = \alpha_s$. In this case, it follows from \eqref{rho_bound_alpha_s} that $|\rho(t)| < \nu$ for all $t \geq 0$. Substituting for $\rho$ from \eqref{position_trnsformed_plane}, we get $|{\alpha(r + \alpha)}/({1 + \alpha r})| < \nu \implies |r + \alpha| < \nu |r + ({1}/{\alpha})|, \forall t \geq 0$. Now, squaring on both sides gives $(r + \alpha)(r^\ast + \alpha) < \nu^2(r + ({1}/{\alpha}))(r^\ast + ({1}/{\alpha}))$. Rearranging terms yields $|r|^2 + \alpha(r + r^\ast) + \alpha^2 < \nu^2(|r|^2 + ({1}/{\alpha})(r + r^\ast) + ({1}/{\alpha^2}))$, implying that 
\begin{equation}\label{bound_equation_1}
 (1 - \nu^2)|r|^2 + \frac{\alpha^2 - \nu^2}{\alpha}(r + r^\ast) + \frac{\alpha^4 - \nu^2}{\alpha^2} < 0, \forall t \geq 0.
\end{equation}
Note that
\begin{equation} \label{nu}
1 - \nu^2 = 	\frac{\mu^2 - (\lambda + \alpha)^2}{\mu^2},	
\end{equation}
where substituting $\mu^2 = (\lambda + \alpha)(\lambda + \beta)$ from \eqref{mu_square}, along with $\beta = 1/\alpha$, we get 
\begin{equation*}
	1 - \nu^2 = 	\frac{(\lambda+\alpha)(1 + \alpha \lambda) - \alpha(\lambda+\alpha)^2}{(\lambda+\alpha)(1 + \alpha \lambda)} = \frac{1 - \alpha^2}{1 + \alpha \lambda}.
\end{equation*}
Clearly, $1 - \nu^2 > 0$ for $\alpha = \alpha_s$, since $1 - \alpha^2 > 0$ and $1 + \alpha\lambda > 0$ (see Theorem~\ref{thm_mapping of regions between circles}). Hence, dividing by $1 - \nu^2$, \eqref{bound_equation_1} becomes
\begin{equation}\label{r_inequality}
|r|^2 + \frac{\alpha^2 - \nu^2}{\alpha(1 - \nu^2)}(r + r^\ast) + \frac{\alpha^4 - \nu^2}{\alpha^2(1 - \nu^2)} < 0, \forall t \geq 0.
\end{equation}
Upon adding and subtracting the square of the term $\frac{\alpha^2 - \nu^2}{\alpha(1 - \nu^2)}$, \eqref{r_inequality} becomes
\begin{align*}
	&|r|^2 + \frac{\alpha^2 - \nu^2}{\alpha(1 - \nu^2)}(r + r^\ast) + \left[\frac{\alpha^2 - \nu^2}{\alpha(1 - \nu^2)}\right]^2 \\
	& \qquad - \left[\frac{\alpha^2 - \nu^2}{\alpha(1 - \nu^2)}\right]^2 + \frac{\alpha^4 - \nu^2}{\alpha^2(1 - \nu^2)} < 0, \forall t \geq 0.
\end{align*}	
After combining the last two terms, the above inequality is simplified to:  	
\begin{equation*}
\left| r + \frac{\alpha^2 - \nu^2}{\alpha(1 - \nu^2)}\right|^2 < \left|\frac{\nu(1 - \alpha^2)}{\alpha(1 - \nu^2)}\right|^2, \forall t \geq 0.
\end{equation*}
In other words, it follows that
\begin{equation}\label{r_bound}
	\left|r(t) + \frac{\alpha^2 - \nu^2}{\alpha(1 - \nu^2)}\right| < \left|\frac{\nu(1 - \alpha^2)}{\alpha(1 - \nu^2)}\right|, \ \forall t \geq 0.
\end{equation}
Substituting for $\nu$, the center of circle \eqref{r_bound} is obtained as $\frac{\alpha^2 - \nu^2}{\alpha(1 - \nu^2)} = \frac{\mu^2 \alpha^2 - (\lambda + \alpha)^2 }{\alpha\mu^2 - \alpha(\lambda + \alpha)^2}$. Now, again replacing $\mu^2$ from \eqref{mu_square} with $\beta = 1/\alpha$ and simplifying, we get	$\frac{\alpha^2 - \nu^2}{\alpha(1 - \nu^2)} = \frac{\alpha(\lambda + \alpha)(1 + \alpha \lambda) - (\lambda + \alpha)^2}{(\lambda + \alpha)(1 + \alpha\lambda) - \alpha (\lambda + \alpha)^2} =  -\lambda$. Similarly, one can easily verify, after substituting for $\nu$ and $\mu^2$, that the radius of circle \eqref{r_bound} is $\left|\frac{\nu(1 - \alpha^2)}{\alpha(1 - \nu^2)}\right| = \frac{|\mu(1 - \alpha^2)(\lambda + \alpha)|}{|(1 - \alpha^2)(\lambda + \mu)|} = \mu$. Using these terms in \eqref{r_bound}, it holds that $|r(t) - \lambda| < \mu, \forall t \geq 0$, satisfying the requirement in Problem~\ref{problem}. 

\item [Case II] If $\alpha = \alpha_{\ell}$. In this case, it holds from \eqref{rho_bound_alpha_l} that $\nu < |\rho(t)|$ for all $t \geq 0$. Analogous to \eqref{bound_equation_1}, we obtain  
\begin{equation}
\label{bound_equation_2} (1 - \nu^2)|r|^2 + \frac{\alpha^2 - \nu^2}{\alpha}(r + r^\ast) + \frac{\alpha^4 - \nu^2}{\alpha^2} > 0, \forall t \geq 0,
\end{equation}
where please note the reversal of the inequity sign. From \eqref{nu}, $1 - \nu^2 = \frac{1 - \alpha^2}{1 + \alpha \lambda} < 0$ for $\alpha = \alpha_{\ell}$, since $1 - \alpha^2 < 0$ and $1 + \alpha \lambda > 0$ (see Theorem~\ref{thm_mapping of regions between circles}). Therefore, the division by $1 - \nu^2$ in \eqref{bound_equation_2} results in the same inequality as in \eqref{r_inequality}. Consequently, it follows that $|r(t) - \lambda| < \mu, \forall t \geq 0$, again satisfying the requirement in Problem~\ref{problem}.    
\end{enumerate}
\end{enumerate}
This concludes the proof.  
\end{proof}

We now establish a connection between the parameters in the two planes, followed by a relationship between the control laws $\omega$ and $\Omega$ in the next section. This is essential since the control $\omega$ is required for implementation in the actual plane.

\section{Relating Actual and Transformed Planes Parameters and Designing Actual Control Law}\label{section_5}
In this section, we focus on establishing the relationship between speeds $v$ and $|\dot{\rho}|$, the heading angles $\theta$ and $\gamma$, and finally the controllers $\omega$ and $\Omega$. Note that these relations can be established either way by considering the map $f$ in \eqref{mobius_transformation} or its inverse $f^{-1}$ (discussed later in detail). Since both provide a different viewpoint and insights into these relations, we discuss both of these, but with detailed derivations only in the former case for brevity.

\subsection{Transformed Plane in Terms of Actual Plane Parameters}\label{sec_mobius_map_relations}
In this subsection, we proceed with the relation \eqref{position_trnsformed_plane}, obtained by directly applying the M\"{o}bius transformation \eqref{mobius_transformation} on the robot's position $r$ in the actual plane, and relate the transformed plane in terms of the actual plane parameters. We begin with the following lemma:

\begin{lem}[$|\rho|$ and $\psi$ in terms of actual plane parameters]\label{lem_rho_modulus_and_psi}
In \eqref{position_trnsformed_plane}, the following relations holds:
	\begin{align}
		\label{rho_modulus}|\rho| &= |\alpha| \sqrt{\frac{\alpha^2+|r|^2+2\alpha|r|\cos\phi}{1 + \alpha^2|r|^2 + 2\alpha|r|\cos\phi}},\\
		\label{rbps}\tan\psi &= \left[\frac{(1-\alpha^2)|r|\sin\phi}{\alpha(1+|r|^2)+(1+\alpha^2)|r|\cos\phi}\right].		
	\end{align}
\end{lem}

\begin{proof}
	From \eqref{position_trnsformed_plane}, we have $|\rho|  = \left|\frac{\alpha(r + \alpha)}{1 +\alpha r}\right| = \frac{|\alpha||r + \alpha|}{|1 + \alpha r|}$. Note that $|r + \alpha| = \big{|} |r|{\rm e}^{i\phi} + \alpha \big{|} = \big{|} (|r|\cos\phi + \alpha) + i |r| \sin\phi \big{|} = \sqrt{(|r|\cos\phi + \alpha)^2 + |r|^2\sin^2\phi} = \sqrt{\alpha^2 + |r|^2 + 2\alpha|r|\cos\phi}$. Similarly, $|1+\alpha r| = \big{|} 1+\alpha |r|{\rm e}^{i\phi} \big{|} = \big{|} (1+\alpha |r|\cos\phi) + i\sin\phi \big{|} = \sqrt{(1+\alpha|r|\cos\phi)^2+\alpha^2|r|^2\sin^2\phi} = \sqrt{1+\alpha^2|r|^2+2\alpha|r|\cos\phi}$. Using these relations, we get \eqref{rho_modulus}.
	
	To prove the second part, we multiply by the conjugate $(1+\alpha r^\ast)$ in the numerator and denominator of \eqref{position_trnsformed_plane} to obtain	
	\begin{equation}\label{rho_complex_conjugate_multiply}
		\rho = \frac{\alpha(r + \alpha)}{1 + \alpha r} \times \frac{1 + \alpha r^\ast}{1 + \alpha r^\ast} = \frac{\alpha[\alpha(1 + |r|^2) + r + \alpha^2 r^\ast]}{(1 + \alpha r)(1 + \alpha r^\ast)}.
	\end{equation}	
	It can be derived that $(1 + \alpha r)(1 + \alpha r^\ast) = 1 + \alpha^2 |r|^2 + \alpha(r + r^\ast) = 1+\alpha^2|r|^2+2\alpha|r|\cos\phi = |1 + \alpha r|^2$, using which, together with substituting $r$ and $r^\ast$ from \eqref{position_actual_plane}, \eqref{rho_complex_conjugate_multiply} becomes
	\begin{equation*}	
		\rho  = \frac{\alpha[\alpha(1+|r|^2)+(1+\alpha^2)|r|\cos\phi + i(1 - \alpha^2)|r|\sin\phi]}{|1 + \alpha r|^2}.
	\end{equation*}	
	Now, substituting for $\rho$ from \eqref{position_trnsformed_plane} and comparing real and imaginary parts, we obtain:
	\begin{subequations}\label{csr}
		\begin{align}
			|\rho|\cos\psi & = \frac{\alpha[\alpha(1+|r|^2)+(1+\alpha^2)|r|\cos\phi]}{|1 + \alpha r|^2},\\
			|\rho|\sin\psi & = \frac{\alpha(1-\alpha^2)|r|\sin\phi}{|1 + \alpha r|^2}.
		\end{align}	
	\end{subequations}
	Using \eqref{csr}, $\tan\psi$ is obtained as in \eqref{rbps}. Hence, proved.  
\end{proof}

Next, we establish a relation between the heading angles $\gamma$ and $\theta$ in the transformed and actual planes, respectively. This is one of the crucial relations in connecting $\omega$ and $\Omega$. 

\begin{lem}[Relationship between $\gamma$ and $\theta$]\label{lem_heading_angle_relation_transformed_plane}
Consider the actual and transformed robot models \eqref{dyn} and \eqref{tdyn} related via the M\"{o}bius transformation \eqref{mobius_transformation}. The heading angles $\gamma$ and $\theta$ are related as:
	\begin{equation}\label{angle_relation_transformed_plane}
		\gamma = \theta - \chi \pmod{\pi},	
	\end{equation}
	where,
	\begin{equation}\label{chi}
		\chi = \arctan \left( \frac{f_2}{f_1} \right), 
	\end{equation}
	with
	\begin{subequations}\label{f1f2}
		\begin{align}	
			f_1 & = 1+2\alpha|r|\cos\phi+\alpha^2|r|^2\cos2\phi,\\
			f_2 & = 2\alpha|r|\sin\phi +\alpha^2|r|^2\sin 2 \phi.	
		\end{align}
	\end{subequations}
	
\end{lem}

\begin{proof}
From \eqref{position_trnsformed_plane}, the time-derivative of $\rho$ (similar to \eqref{mobius_transform_derivative}) is obtained as
	\begin{equation}\label{rho_dot}
		\dot{\rho}	= \frac{\alpha (1-\alpha^2)\dot{r}}{(1 + \alpha r)^2} = \frac{\alpha (1-\alpha^2)v {\rm e}^{i\theta}}{(1 + \alpha r)^2}.
	\end{equation}
On multiplying by the conjugate square term $(1 + \alpha r^\ast)^2$ in the numerator and denominator of \eqref{rho_dot}, yields
	\begin{equation*}
		\dot{\rho} 	= \frac{\alpha (1-\alpha^2)v {\rm e}^{i\theta}}{(1 + \alpha r)^2} \times \frac{(1 + \alpha r^\ast)^2}{(1 + \alpha r^\ast)^2} = \frac{ \alpha (1-\alpha^2)v {\rm e}^{i\theta}(1 + \alpha r^\ast)^2}{|1+\alpha r|^4}.
	\end{equation*}
	Now, substituting for $r^\ast$ and $\dot{\rho}$ from \eqref{position_actual_plane} and \eqref{tdyn}, respectively, and simplifying, we get
	\begin{equation}\label{e_i_gamma_initial}
		|\dot{\rho}|{\rm e}^{i\gamma} = \frac{\alpha (1-\alpha^2)v[{\rm e}^{i\theta} + 2\alpha|r|{\rm e}^{i(\theta - \phi)} + \alpha^2 |r|^2 {\rm e}^{i(\theta - 2\phi)}]}{|1+\alpha r|^4}.
	\end{equation}
Comparing real and imaginary parts in \eqref{e_i_gamma_initial}, one can obtain
	\begin{equation}\label{tan_gamma_1}
		\tan\gamma = \frac{\sin\theta+2\alpha|r|\sin(\theta-\phi) + \alpha^2|r|^2\sin(\theta-2\phi) }{\cos\theta+2\alpha|r|\cos(\theta-\phi) +\alpha^2|r|^2\cos(\theta-2\phi)}.
	\end{equation}
	Using the trigonometric identities $\sin(A-B) = \sin A\cos B - \cos A \sin B$ and $\cos(A-B) = \cos A \cos B + \sin A \sin B$, \eqref{tan_gamma_1} can be expressed as
	\begin{equation}\label{tan_gamma}
		\tan\gamma = \frac{f_1\sin\theta-f_2\cos\theta}{f_1\cos\theta+f_2\sin\theta},
	\end{equation}
	where $f_1, f_2$ are as defined in \eqref{f1f2}. Dividing numerator and denominator of \eqref{tan_gamma} by $\sqrt{f_1^2+f_2^2}$ and representing a right angle triangle with base $f_1$, altitude $f_2$ and hypotenuse $\sqrt{f_1^2+f_2^2}$ with $\chi = \arctan(f_2/f_1)$, it can be written that $\cos\chi = f_1/\sqrt{f_1^2+f_2^2}$ and $\sin\chi = f_2/\sqrt{f_1^2+f_2^2}$, and hence, 
	\begin{equation*}
		\tan\gamma = \frac{\sin(\theta-\chi)}{\cos(\theta-\chi)} = \tan(\theta-\chi),
	\end{equation*}
	which gives $\gamma = n\pi + (\theta - \chi), \ n \in \mathbb{Z}$. Hence, proved. 
\end{proof}

\begin{lem}[Time-derivative of $\chi$]\label{lem_chi_dot}
	The time-derivative of $\chi$ in \eqref{chi} is given by
	\begin{equation}\label{chi_dot_final}
		\dot{\chi} = 2\alpha v\left[\frac{\sin\theta+\alpha|r|\sin(\theta-\phi)}{|1+\alpha r|^2}\right].
	\end{equation}
\end{lem}

\begin{proof}
Please refer to Appendix~\ref{appendix_B}. 
\end{proof}

Exploiting Lemma~\ref{lem_heading_angle_relation_transformed_plane} and Lemma~\ref{lem_chi_dot}, we relate $\omega$ and $\Omega$ in the following theorem and represent $\omega$ solely in terms of the actual plane parameters. 

\begin{thm}[Control law $\omega(r, \theta)$]\label{thm_control_relations_actual_plane_parameters}
	Consider the actual and transformed robot models \eqref{dyn} and \eqref{tdyn}, respectively. The actual control law $\omega$ is related to the transformed control law $\Omega$ as follows:
	\begin{equation}\label{control_relations_actual_plane_parameters}
		\omega = \Omega + 2\alpha v\left[\frac{\sin\theta+\alpha|r|\sin(\theta-\phi)}{|1+\alpha r|^2}\right].
	\end{equation}
Further, \eqref{control_relations_actual_plane_parameters} can be expressed in terms of the actual plane parameters as
\begin{align}\label{control_law_in_actual_plane_parameters_final}
	\nonumber  \omega(r, \theta) &= \frac{v}{\sigma}\left|\frac{\alpha (1-\alpha^2)}{(1+\alpha r)^2}\right|  +  \frac{\kappa}{\sigma} \left[\frac{\mathcal{P}(r, \theta)}{\delta^2_T - |\mathcal{E}(r, \theta)|^2}\right]\\
	& \qquad + 2\alpha v\left[\frac{\sin\theta+\alpha|r|\sin(\theta-\phi)}{|1+\alpha r|^2}\right],
\end{align}
where, 
\begin{align}
	\label{inner_produt_term}\mathcal{P}(r, \theta) &= \left \langle \frac{\alpha(r + \alpha)}{1+\alpha r} \; , \; \sgn(\Delta)\frac{1 + \alpha r^\ast}{1 + \alpha r} {\rm e}^{i\theta} \right \rangle\\
	\label{error_actual_plane_parameters}\mathcal{E}(r, \theta)  &= \frac{\alpha(r + \alpha)}{1 + \alpha r} + i\sigma \sgn(\Delta)\left[\frac{1 + \alpha r^\ast}{1 + \alpha r}\right]{\rm e}^{i\theta}, 	
\end{align}
and $\Delta = \beta - \alpha$ is the difference between the roots $\alpha$ and $\beta$ of \eqref{mobius_roots}.	
\end{thm}

Note that the angle $\phi$ is inherited in $r$ (see \eqref{position_actual_plane}), therefore, the argument of $\omega$ in \eqref{control_law_in_actual_plane_parameters_final} only uses $(r, \theta)$. In the same spirit, $\omega(\rho, \gamma)$ is used in Theorem~\ref{thm_control_relations_transformed_plane_parameters} in the next subsection.

\begin{proof}
	From Lemma~\ref{lem_heading_angle_relation_transformed_plane}, the time-derivative of \eqref{angle_relation_transformed_plane} is
\begin{equation}\label{gamma_dot}
	\dot{\gamma} = \dot{\theta} -\dot{\chi}.
\end{equation}
Now, substituting for $\dot{\theta}$, $\dot{\gamma}$ and $\dot{\chi}$ from \eqref{dyn}, \eqref{tdyn}, and \eqref{chi_dot_final}, respectively, we get the required result \eqref{control_relations_actual_plane_parameters}.

The second statement is proven as follows. By taking modulus of \eqref{rho_dot}, we can write
\begin{equation}\label{transformed_speed_in_actual_plane_parameters}
	|\dot{\rho}| = \left|\frac{\alpha (1-\alpha^2)}{(1+\alpha r)^2}\right|v.
\end{equation} 
Further, using \eqref{rho_dot} and \eqref{transformed_speed_in_actual_plane_parameters}, the unit vector ${\rm e}^{i\gamma}$ in \eqref{tdyn} can be expressed in terms of actual plane parameters as:
\begin{equation}\label{e_gamma_final}
	{\rm e}^{i\gamma} {=} \frac{\dot{\rho}}{|\dot{\rho}|} {=} \frac{\alpha(1 - \alpha^2)}{|\alpha(1 - \alpha^2)|}\frac{|1 + \alpha r|^2}{(1 + \alpha r)^2} {\rm e}^{i\theta} {=} \sgn(\Delta)\left[\frac{1 + \alpha r^\ast}{1 + \alpha r}\right]  {\rm e}^{i\theta},
\end{equation} 
where we used the fact that ${\alpha(1 - \alpha^2)}/{|\alpha(1 - \alpha^2)|} = (\beta - \alpha)/|\beta - \alpha| = \Delta/|\Delta|$ after replacing 1 by the product $\alpha\beta$ and simplifying. Next, using \eqref{position_trnsformed_plane} and \eqref{e_gamma_final}, the inner product $\langle \rho, {\rm e}^{i\gamma} \rangle$ can be readily expressed in terms of the actual plane parameters as $\mathcal{P}(r, \theta)$ in \eqref{inner_produt_term}, and the error $\mathcal{E}$ in \eqref{error_transformed_plane} can be written in terms of actual plane parameters as \eqref{error_actual_plane_parameters}. Now, the result immediately follows from \eqref{control_law_transformed} and \eqref{control_relations_actual_plane_parameters}.   
\end{proof}

Note that our efforts so far have been on expressing transformed plane parameters in terms of the actual plane parameters. However, it is equally important to represent the control law $\omega$ solely in terms of the transformed plane parameters, specifically, for the ease of implementation. In this direction, the role of inverse M\"{o}bius map $f^{-1}$ (defined by \eqref{imt}) becomes important since it allows us to express the actual plane parameters in terms of the transformed plane parameters. This is the topic of the next subsection.   

\subsection{Actual Plane in Terms of Transformed Plane Parameters}
The inverse M\"{o}bius map gives the robot's position in the actual plane in terms of its position in the transformed plane and is obtained from \eqref{position_trnsformed_plane} as: 
\begin{equation} \label{inverse_position_in_actual_plane}
	\rho \mapsto r = f^{-1}(\rho) = \frac{\alpha^2-\rho}{\alpha(\rho-1)}.
\end{equation}
Note that \eqref{inverse_position_in_actual_plane} has a singularity at $\rho = 1$. However, we will demonstrate in the following lemma that it remains well-defined within the context of the problem under consideration.  
 
\begin{lem}[Existence of inverse M\"{o}bius map]
The inverse M\"{o}bius transformation \eqref{inverse_position_in_actual_plane} is a well-defined mapping, under the conditions given in Theorem~\ref{thm_main_theorem}.  
\end{lem} 

\begin{proof}
We prove it by contradiction. If $\rho = 1$, it follows from \eqref{position_trnsformed_plane} that $1 + \alpha r = \alpha(r + \alpha) \implies \alpha^2 = 1 \implies \alpha = \pm 1$, which contradicts our hypothesis in Theorem~\ref{thm_main_theorem} (see Corollary~\ref{cor_quadractic_roots}). 
\end{proof}

Utilizing \eqref{inverse_position_in_actual_plane}, analogous relations as in Subsection~\ref{sec_mobius_map_relations} can be established. In this direction, $|r|$ and $\phi$ can be  expressed in terms of transformed plane parameters $|\rho|$ and $\psi$, similar to Lemma~\ref{lem_rho_modulus_and_psi}, as follows: 
	\begin{align}
	\label{r_modulus}|r| &= \sqrt{\frac{\alpha^2 + \beta^2|\rho|^2 - 2|\rho|\cos\psi}{1 + |\rho|^2 - 2|\rho|\cos\psi}},\\
	\label{tan_phi}\tan\phi &= \left[\frac{(1-\alpha^2)|\rho|\sin\psi}{(1+\alpha^2)|\rho|\cos\psi - (\alpha^2 + |\rho|^2)}\right].		
\end{align}
Furthermore, analogous to Lemma~\ref{lem_heading_angle_relation_transformed_plane}, we have the following lemma considering inverse map \eqref{inverse_position_in_actual_plane}.

\begin{lem}[Relationship between $\theta$ and $\gamma$]\label{lem_heading_angle_relation_actual_plane}
Consider the actual and transformed robot models \eqref{dyn} and \eqref{tdyn} related via the inverse M\"{o}bius transformation \eqref{inverse_position_in_actual_plane}. The heading angles $\theta$ and $\gamma$ are related as: 
	\begin{equation}\label{angle_relation_actual_plane}
		\theta = \gamma - \xi \pmod{\pi},	
	\end{equation}
	where,
	\begin{equation}\label{xi}
		\xi  = \arctan\left(\frac{g_2}{g_1}\right),
	\end{equation}
	with
	\begin{subequations}\label{g1g2}
		\begin{align}
			\label{g1} g_1 &= 1+|\rho|^2\cos2\psi-2|\rho|\cos\psi,\\
			\label{g2} g_2 &= |\rho|^2\sin2\psi-2|\rho|\sin\psi.
		\end{align}
	\end{subequations}
\end{lem}

\begin{proof}
Similar to \eqref{rho_dot}, time-derivative of \eqref{inverse_position_in_actual_plane} is obtained as:
\begin{equation}\label{rdot}
\dot{r}	= \left(\frac{1-\alpha^2}{\alpha}\right)\left[\frac{\dot{\rho}}{(\rho-1)^2}\right] = \left(\frac{1-\alpha^2}{\alpha}\right)\left[\frac{|\rho|{\rm e}^{i \gamma}}{(\rho-1)^2}\right].
\end{equation}	
Now, proceeding along the similar lines as in Lemma~\ref{lem_heading_angle_relation_transformed_plane}, the proof readily follows, and is omitted. 
\end{proof}

\begin{lem}[Time-derivative of $\xi$]\label{lem_xi_dot}
The time-derivative of $\xi$ in \eqref{xi} is given by
	\begin{equation}\label{xi_dot_final}
		\dot{\xi} = \left(\frac{2\alpha v}{1 - \alpha^2}\right)[|\rho|\sin(\gamma-\psi)-\sin\gamma].
	\end{equation}
\end{lem}

\begin{proof}
The proof follows along the similar lines as in Lemma~\ref{lem_chi_dot}, and hence, omitted for brevity. 
\end{proof}

Exploiting Lemma~\ref{lem_heading_angle_relation_actual_plane} and Lemma~\ref{lem_xi_dot}, we relate $\omega$ and $\Omega$ in the following theorem and represent $\omega$ solely in terms of the transformed plane parameters. 

\begin{thm}[Control law $\omega(\rho, \gamma)$]\label{thm_control_relations_transformed_plane_parameters}
	Consider the actual and transformed robot models \eqref{dyn} and \eqref{tdyn}, respectively. The actual control input $\omega$ is related to the transformed control $\Omega$ as follows:
	\begin{equation}\label{control_relations_transformed_plane_parameters}
		\omega = \Omega - \left(\frac{ 2 \alpha v}{1 - \alpha^2}\right)[|\rho|\sin(\gamma-\psi)-\sin\gamma].
	\end{equation}
	Further, \eqref{control_relations_transformed_plane_parameters} can be expressed in terms of transformed plane parameters as
	\begin{align}\label{control_law_in_transformed_plane_parameters_final}
		\nonumber  \omega(\rho, \gamma) &= \frac{v}{\sigma}\left|\frac{\alpha(\rho - 1)^2}{1 - \alpha^2}\right|  + \frac{\kappa}{\sigma} \left[\frac{ \left \langle \rho, {\rm e}^{i\gamma} \right \rangle }{\delta_T^2 - |\mathcal{E}(\rho, \gamma)|^2}\right] \\
		& \qquad - \left(\frac{ 2 \alpha v}{1 - \alpha^2}\right)\left[|\rho|\sin(\gamma-\psi)-\sin\gamma\right].
	\end{align}	
\end{thm}

\begin{proof}
The proof of the first part follows by taking time-derivative of \eqref{angle_relation_actual_plane} and then substituting for $\dot{\theta}$, $\dot{\gamma}$ and $\dot{\xi}$ from \eqref{dyn}, \eqref{tdyn} and \eqref{xi_dot_final}, respectively.  
	
To prove the second statement, we use \eqref{rdot} to obtain that
	\begin{equation}\label{transformed_speed_in_transformed_plane_parameters}
		|\dot{\rho}|  = \left|\frac{\alpha(\rho - 1)^2}{1 - \alpha^2}\right|v.
	\end{equation} 	
	Now, substituting $\Omega$ from \eqref{control_law_transformed} into \eqref{control_relations_transformed_plane_parameters} and using \eqref{transformed_speed_in_transformed_plane_parameters}, we get \eqref{control_law_in_transformed_plane_parameters_final}, where $\mathcal{E}(\rho, \gamma)$ is given by \eqref{error_transformed_plane}.   
\end{proof}

\begin{remark}[Feasible initial conditions]
The validity of control \eqref{control_law_transformed} relies on the condition that $|\mathcal{E}(0)| < \delta_T$ must be satisfied in the transformed plane (see Theorem~\ref{thm_control_law_transformed}). It is necessary to map this requirement on the initial position and heading angle of the robot in the actual plane. In this direction, it follows from \eqref{error_actual_plane_parameters} that the condition
\begin{equation}\label{initial_condition_inequality}
	\left|\frac{\alpha(r(0) + \alpha)}{1 + \alpha r(0)} + i\sigma \sgn(\Delta)\left[\frac{1 + \alpha r^\ast(0)}{1 + \alpha r(0)}\right]{\rm e}^{i\theta(0)}\right| < \delta_T,
\end{equation}     
must be satisfied for appropriately chosen initial position $r(0)$ and heading angle $\theta(0)$. Using \eqref{position_actual_plane} and separating the real and imaginary parts in the numerator, \eqref{initial_condition_inequality} reduces to
\begin{equation}\label{initial_conditions}
	\left|\frac{\eta_a(0) + i\eta_b(0)}{1 + \alpha |r(0)|{\rm e}^{i\phi(0)}} \right| < \delta_T \implies \eta_a^2(0) +\eta_b^2(0) < \delta^2_T\eta^2(0),
\end{equation}
where $\eta(r) = \sqrt{1 + 2\alpha |r|\cos\phi + \alpha^2|r|^2}$, and 
\begin{subequations}
	\begin{align*}
		\eta_a(r, \theta) & =  \alpha^2 + \alpha|r|\cos\phi - \sigma\sgn(\Delta)(\sin\theta + \alpha|r|\sin(\theta-\phi)),\\
		\eta_b(r, \theta) & = 	\alpha|r|\sin\phi + \sigma\sgn(\Delta)(\cos\theta + \alpha|r|\cos(\theta-\phi)). 
	\end{align*}	
\end{subequations}
Please note that we have used shorthand notations $\eta_a(0)$, $\eta_b(0)$ and $\eta(0)$ for $\eta_a(r(0), \theta(0))$, $\eta_b(r(0), \theta(0))$ and $\eta(r(0))$, respectively. According to Problem~\ref{problem}, once the initial position $r(0)$ is fixed, $\eta(0)$ got fixed and the free states $\theta(0) \in \mathbb{S}^1$ can be easily chosen such that the inequality \eqref{initial_conditions} is satisfied.   
\end{remark}

\section{Stringent Bounds on Post-Design Signals}\label{section_6}
In this section, we obtain conservative bounds on the post-design signals by employing the results of Theorem~\ref{thm_control_law_transformed}. By substituting $\mathcal{E}(r, \theta)$ from \eqref{error_actual_plane_parameters}, the potential $\mathcal{S}$ in \eqref{Lyapunov_function} can be expressed as $\mathcal{S}(r, \theta)$, (i.e., in terms of actual plane parameters $r$ and $\theta$), and hence, $\mathcal{S}(r(0), \theta(0)) \coloneqq \mathcal{S}(0)$ can be obtained. Let the constant $\Theta$ be defined as $\Theta \coloneqq \sqrt{1-{\rm e}^{-2\mathcal{S}(0)}}$ for the given initial states. Clearly, $0 \leq \Theta \leq 1$ with $\Theta = 0$ precisely at the equilibrium point where $\mathcal{S} = 0$, and $\Theta = 1$ if $\mathcal{S}(0) \to \infty$. Relying on $\Theta$, we further define the convex combination of two radii $|\alpha|$ and $|(\lambda + \alpha)/\mu|$ (of the circles in Fig.~\ref{fig_problem_transformed_plane}) as:  
\begin{equation}\label{convex_combination_radii_positive}
\nu_{\Theta} = (1 - \Theta)|\alpha| + \Theta\left|\frac{\lambda + \alpha}{\mu}\right|. 	
\end{equation}  
By replacing $\Theta$ by $-\Theta$ in \eqref{convex_combination_radii_positive}, we define a new parameter
\begin{equation}\label{convex_combination_radii_negative}
\nu_{-\Theta} = (1 + \Theta)|\alpha| - \Theta\left|\frac{\lambda + \alpha}{\mu}\right|, 	
\end{equation}  
that will be helpful in the subsequent analysis. We have the following theorem: 

\begin{thm}[Stringent bounds on post-design signals]\label{thm6}
Under the conditions given in Theorem~\ref{thm_control_law_transformed}, the following inequalities hold in the transformed and actual planes:
\begin{itemize}[leftmargin=*]
	\item Transformed plane 
\begin{enumerate}[leftmargin=*]
\item[(a)] The absolute values of error $\mathcal{E}(t)$ and position $\rho(t)$	are bounded by
\begin{equation}\label{strict_bound_error}
|\mathcal{E}(t)| \leq \delta_T\Theta, \quad \forall t \geq 0, 	
\end{equation}
\begin{equation}\label{strict_bound_rho}
|\rho(t)| \in 
\begin{cases}
	\left[\nu_{-\Theta}, \nu_{\Theta}\right], & \text{if } \alpha = \alpha_s\\
	\left[\nu_{\Theta}, \nu_{-\Theta}\right], & \text{if } \alpha = \alpha_{\ell}
\end{cases}, \quad \forall t \geq 0.
\end{equation}
\item[(b)] The control law $\Omega$ in \eqref{control_law_transformed} is uniformly bounded by
\begin{equation}\label{stric_bound_Omega_1}
|\Omega| \leq \frac{v}{\sigma}\left|\frac{\alpha}{1 - \alpha^2}\right|(1 + \nu_{\Theta})^2  + \frac{\kappa}{\sigma} \left[\frac{\nu_{\Theta}}{\delta_T^2(1 - \Theta)^2}\right],
\end{equation}
if $\alpha = \alpha_s$, and 
\begin{equation}\label{stric_bound_Omega_2}
|\Omega| \leq \frac{v}{\sigma}\left|\frac{\alpha}{1 - \alpha^2}\right|(1 + \nu_{-\Theta})^2  + \frac{\kappa}{\sigma} \left[\frac{\nu_{-\Theta}}{\delta_T^2(1 - \Theta)^2}\right], 
\end{equation}
if $\alpha = \alpha_{\ell}$. 
\end{enumerate}
\item Actual plane
\begin{enumerate}[leftmargin=*]
	\item[(a)] The actual position $r(t)$ remains bounded within the circle
	\begin{equation}\label{r_bound_new}
		\left|r(t) + \frac{\alpha^2 - \nu^2_\Theta}{\alpha(1 - \nu^2_\Theta)}\right| \leq \left|\frac{\nu_\Theta(1 - \alpha^2)}{\alpha(1 - \nu^2_\Theta)}\right|,
	\end{equation}
irrespective of the roots $\alpha_s$ or $\alpha_{\ell}$ for all $t \geq 0$. 
	\item[(b)] The actual control $\omega$ in \eqref{control_law_in_transformed_plane_parameters_final} is uniformly bounded by
	\begin{align}
		\nonumber |\omega| &\leq \frac{v}{\sigma}\left|\frac{\alpha}{1 - \alpha^2}\right| \left[(1 + \sigma + \nu_{\Theta})^2 - \sigma^2 \right] \\
	\label{strict_bound_omega_1}& \qquad \qquad + \frac{\kappa}{\sigma}\left[\frac{\nu_{\Theta}}{\delta_T^2(1 - \Theta)^2}\right],	\quad \text{if~} \alpha = \alpha_s,\\
	\nonumber |\omega| &\leq \frac{v}{\sigma}\left|\frac{\alpha}{1 - \alpha^2}\right| \left[(1 + \sigma + \nu_{-\Theta})^2 - \sigma^2 \right] \\
	\label{strict_bound_omega_2} & \qquad \qquad + \frac{\kappa}{\sigma}\left[\frac{\nu_{-\Theta}}{\delta_T^2(1 - \Theta)^2}\right],	\quad \text{if~} \alpha = \alpha_{\ell}. 
	\end{align}	
\end{enumerate}
\end{itemize}
\end{thm}

\begin{proof}
We first prove the results in transformed plane. 
\begin{enumerate}[leftmargin=*]
\item[(a)] According to Theorem~\ref{thm_control_law_transformed}, since $|\mathcal{E}(t)| < \delta_T$ and $\dot{\mathcal{S}} \leq 0$, $\mathcal{S}(\mathcal{E}) \leq \mathcal{S}(0), \forall t \geq 0$. Replacing $\mathcal{S}(\mathcal{E})$ from \eqref{Lyapunov_function}, we have $\frac{1}{2}  \ln \left(\frac{\delta_T^2}{\delta_T^2 - |\mathcal{E}(t)|^2}\right) \leq \mathcal{S}(0) \implies \ln \left(\frac{\delta_T^2}{\delta_T^2 - |\mathcal{E}(t)|^2}\right) \leq 2\mathcal{S}(0), \forall t\geq 0$. Taking exponential on both sides, we get ${\delta_T^2}/(\delta_T^2 - |\mathcal{E}(t)|^2) \leq {\rm e}^{2\mathcal{S}(0)} \implies \delta_T^2 \leq {\rm e}^{2\mathcal{S}(0)} \left(\delta_T^2 - |\mathcal{E}(t)|^2\right) \implies |\mathcal{E}(t)| \leq \delta_T \sqrt{1-{\rm e}^{-2\mathcal{S}(0)}} = \delta_T\Theta,  \forall t \geq 0$. Further, substituting $\mathcal{E}$ from \eqref{error_transformed_plane}, we have $|\rho(t) + i \sigma {\rm e}^{i\gamma(t)}| \leq \delta_T\Theta$. Similar to the proof of part (b) of Theorem~\ref{thm_control_law_transformed}, it can be obtained that $|\alpha| - \delta_T\Theta \leq |\rho(t)| \leq |\alpha| + \delta_T\Theta, \forall t \geq 0$. Now, substituting $\delta_T$ from \eqref{delta_transformed} for $\alpha_s$ and $\alpha_{\ell}$, one can easily conclude the required result \eqref{strict_bound_rho}. 

\item[(b)] Substituting \eqref{transformed_speed_in_transformed_plane_parameters} into \eqref{control_law_transformed} and taking modulus on both sides, the uniform bound on $\omega$ can be obtained as
\begin{equation}\label{strict_bound_intermediate_1}
	|\Omega| \leq \frac{v}{\sigma}\left|\frac{\alpha(\rho - 1)^2}{1 - \alpha^2}\right|  + \frac{\kappa}{\sigma} \left|\frac{ \left \langle \rho, {\rm e}^{i\gamma} \right \rangle }{\delta_T^2 - |\mathcal{E}(\rho, \gamma)|^2}\right|, 
\end{equation} 
where we used the triangle inequality. It is worth noticing that $|\rho - 1|^2 = (\rho - 1)(\rho^\ast - 1) = |\rho|^2 - 2|\rho|\cos\psi + 1 \leq (1 + |\rho|)^2$, $\left|\left \langle \rho, {\rm e}^{i\gamma} \right \rangle\right| \leq |\rho|$ and $\delta^2_T - |\mathcal{E}|^2 \geq \delta^2_T{\rm e}^{-2\mathcal{S}(0)} = \delta^2_T(1 - \Theta^2)$, using the part (a) above. Applying these, it follows that 
\begin{equation}\label{strict_bound_intermediate_2}
|\Omega| \leq \frac{v}{\sigma}\left|\frac{\alpha}{1 - \alpha^2}\right|(1 + |\rho|)^2  + \frac{\kappa}{\sigma} \left[\frac{|\rho|}{\delta_T^2(1 - \Theta)^2}\right].		 
\end{equation}
Now, employing the upper bounds of $|\rho(t)|$ from \eqref{strict_bound_rho} for both $\alpha_s$ and $\alpha_{\ell}$, the required bounds in \eqref{stric_bound_Omega_1} and \eqref{stric_bound_Omega_2} follow, respectively.
\end{enumerate} 

Next, we prove results in case of actual plane. 

\begin{enumerate}[leftmargin=*]
\item[(a)] In the spirit of Theorem~\ref{thm_problem_1_solution}, we consider upper bound $\nu_{\Theta}$ for $\alpha = \alpha_s$ and lower bound $\nu_{\Theta}$ for $\alpha = \alpha_{\ell}$ in \eqref{strict_bound_rho}, respectively, as only these matters as far as Problems~\ref{problem} is concerned. Nevertheless, the analysis can be easily extended for the remaining bounds analogously, and omitted for brevity. From \eqref{convex_combination_radii_positive}, since $\nu_{\Theta}$ is a convex combination of the two radii, it follows that $|\alpha| \leq \nu_{\Theta} \leq |\frac{\lambda + \alpha}{\mu}|$ if $\alpha = \alpha_s$, and $|\frac{\lambda + \alpha}{\mu}| \leq \nu_{\Theta} \leq |\alpha|$ if $\alpha = \alpha_{\ell}$ (see Fig.~\ref{fig_problem_transformed_plane}). From the former (resp., later) inequality, it can be concluded that $1 - (\frac{\lambda + \alpha}{\mu})^2 \leq 1 - \nu^2_{\Theta} \leq 1 - \alpha^2 \implies  1 - \nu^2 \leq 1 - \nu^2_{\Theta} \leq 1 - \alpha^2$ if $\alpha = \alpha_s$ (resp., $1 - \alpha^2 \leq 1 - \nu^2_{\Theta} \leq 1 - \nu^2$ if $\alpha = \alpha_{\ell}$), where $1 - \nu^2$ is given by \eqref{nu}. As discussed in Theorem~\ref{thm_problem_1_solution}, since $1 - \nu^2 > 0$ for $\alpha = \alpha_s$ (resp., $1 - \nu^2 < 0$ for $\alpha = \alpha_{\ell}$), the preceding inequalities imply that $1 - \nu^2_{\Theta} > 0$ for $\alpha = \alpha_s$ (resp., $1 - \nu^2_{\Theta} < 0$ for $\alpha = \alpha_{\ell}$). Now, following along the similar lines of proof as in Theorem~\ref{thm_problem_1_solution}, it is straightforward to observe that the inequality \eqref{r_bound_new} follows analogously to \eqref{r_bound} for both the cases of $\alpha$.   

\item[(b)] From \eqref{control_relations_transformed_plane_parameters}, the uniform bound on $\omega$ can be obtained as
\begin{equation}\label{strict_bounnd_omega_intermdeditae}
|\omega| \leq |\Omega| + 2v\left|\frac{\alpha}{1 - \alpha^2}\right|\big{|}|\rho|\sin(\gamma-\psi)-\sin\gamma\big{|}.
\end{equation}
Exploiting that $\big{|}|\rho|\sin(\gamma-\psi)-\sin\gamma\big{|} \leq 1 + |\rho|$, along with, \eqref{strict_bound_intermediate_2} and \eqref{strict_bounnd_omega_intermdeditae}, and simplifying, yields 
\begin{equation}\label{strict_bound_intermediate_3}
|\omega| {\leq} \frac{v}{\sigma}\left|\frac{\alpha}{1 - \alpha^2}\right| \left[(1 + \sigma + |\rho|)^2 - \sigma^2 \right]  + \frac{\kappa}{\sigma}\left[\frac{|\rho|}{\delta_T^2(1 - \Theta)^2}\right].	 
\end{equation}
Now, using \eqref{strict_bound_rho} in the same spirit as in the above part (b) in the transformed plane, the required bounds \eqref{strict_bound_omega_1} and \eqref{strict_bound_omega_2} are obtained. 
\end{enumerate}
This concludes the proof. 
\end{proof}

\begin{figure*}[t!]
	\centering{
	\subfigure[Trajectory]{\includegraphics[width=3.5cm]{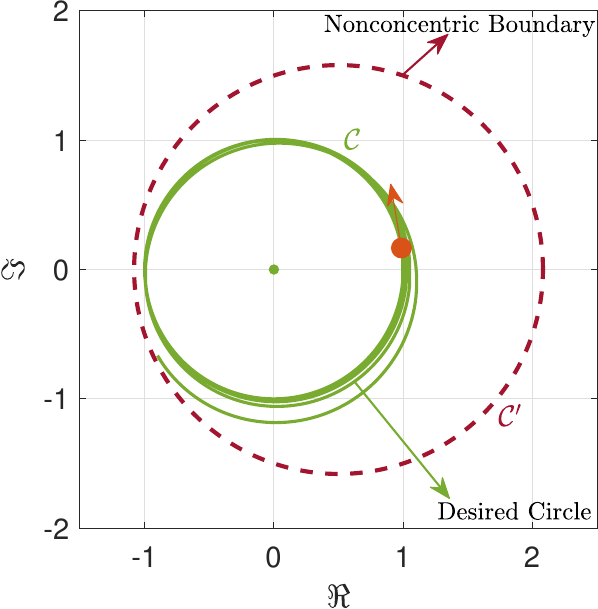}\label{trj_1}}\hspace*{0.3cm}
	\subfigure[Absolute error $|e|$]{\includegraphics[width=3.77cm]{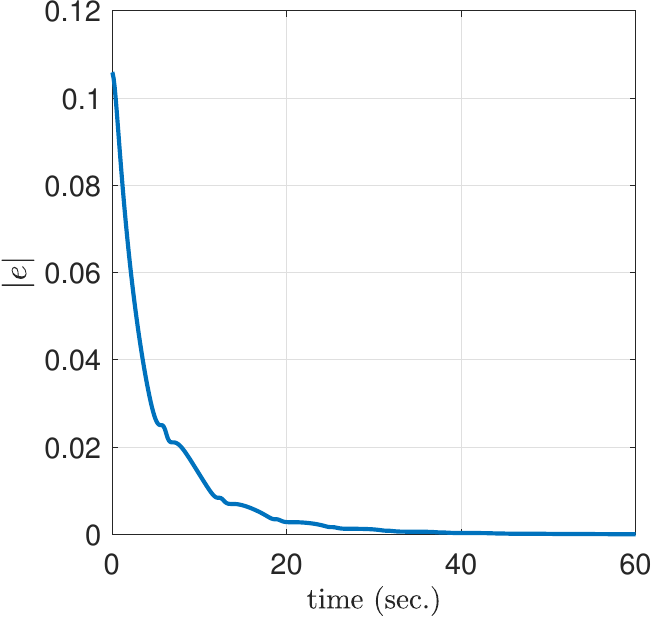}\label{error_1}}\hspace*{0.2cm}
	\subfigure[$\omega$ with $\alpha=\alpha_s$]{\includegraphics[width=3.74cm]{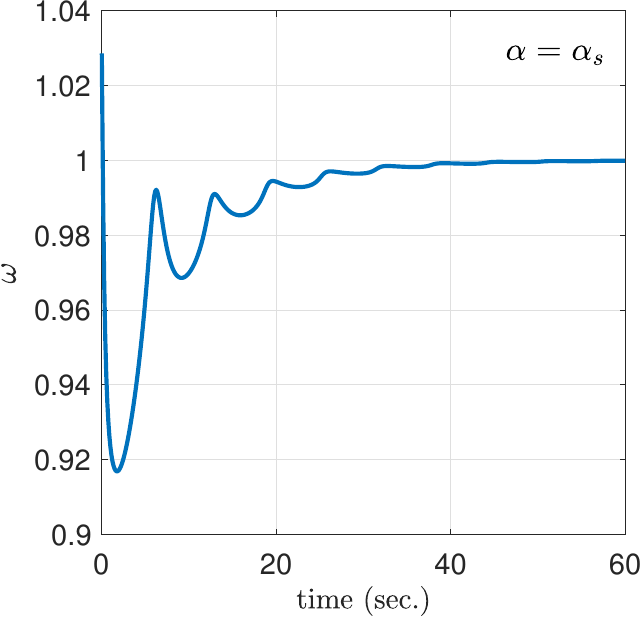}\label{omega_alpha_s}}\hspace*{0.2cm}
	\subfigure[$\omega$ with $\alpha=\alpha_{\ell}$]{\includegraphics[width=3.74cm]{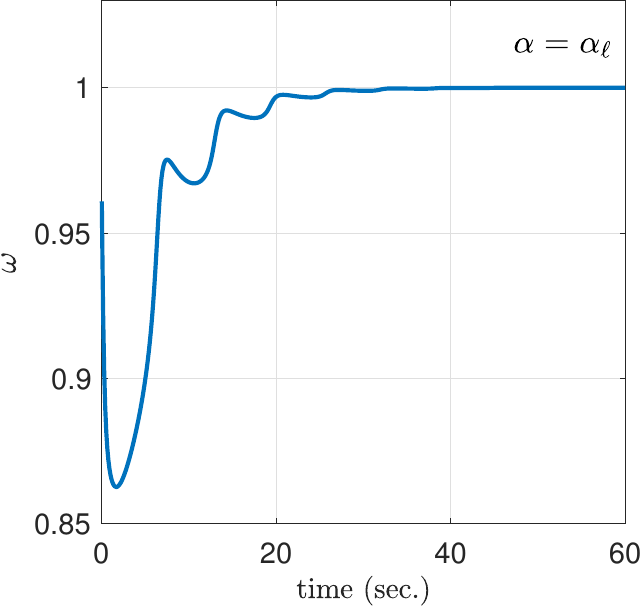}\label{omega_alpha_l}}
}
	\caption{Robot's trajectory and evolution of absolute error and control input with time in the actual plane.}
\label{plot1}
\end{figure*}

\begin{figure*}[t!]
	\centering{
	\subfigure[Trajectory with $\alpha = \alpha_s$]{\includegraphics[width=3.6cm]{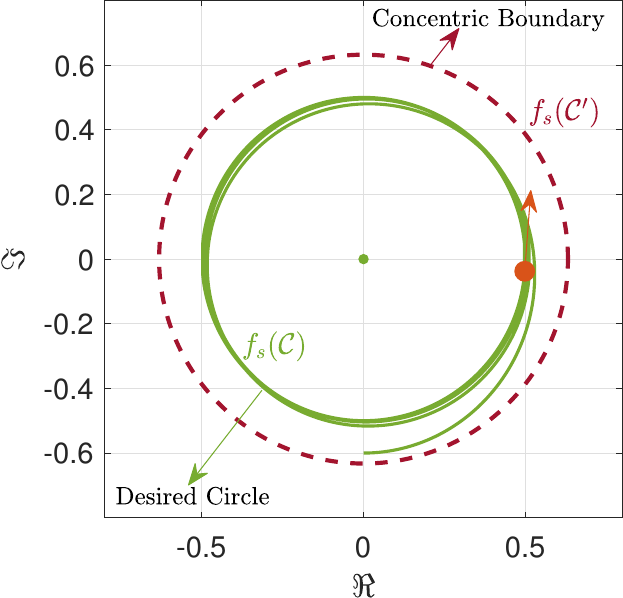}\label{trj_2}}\hspace*{0.3cm}
	\subfigure[Absolute error $|\mathcal{E}|$]{\includegraphics[width=3.76cm]{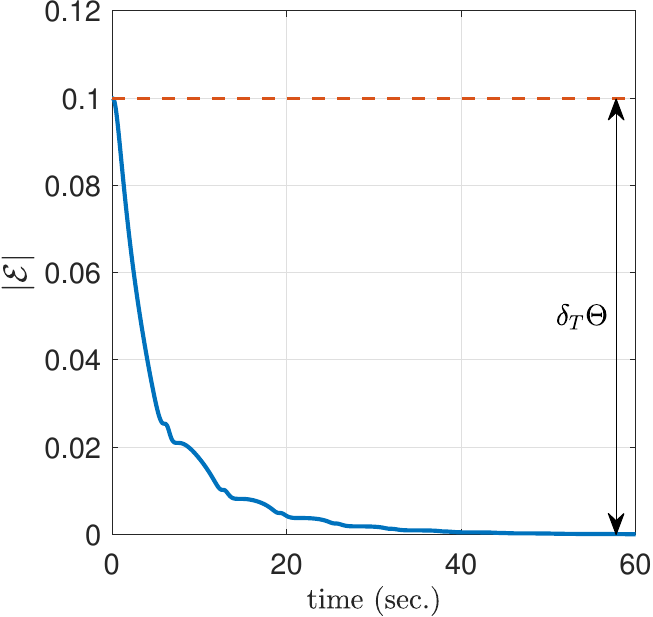}\label{error_2}}\hspace*{0.2cm}
	\subfigure[Control input $\Omega$]{\includegraphics[width=3.73cm]{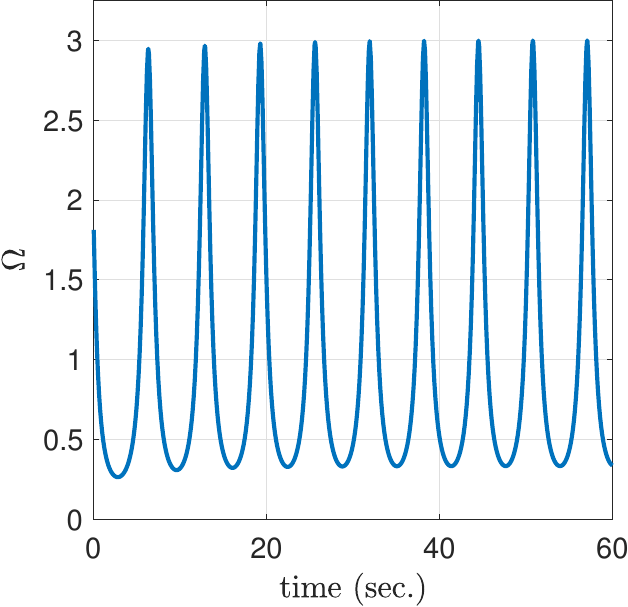}\label{OMEGA_alpha_s}}\hspace*{0.2cm}
	\subfigure[Trajectory with $\alpha = \alpha_{\ell}$]{\includegraphics[width=3.6cm]{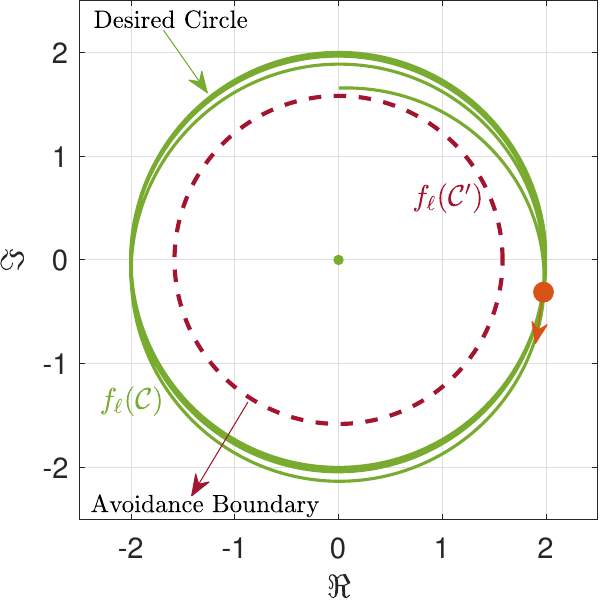}\label{trj_3}}
}
	\caption{Robot's trajectories for both $\alpha_s$ and $\alpha_{\ell}$, and evolution of absolute error and  control input in the transformed plane.}
    \label{plot2}	
\end{figure*}

\begin{figure*}[t!]
	\centering{
	\subfigure[Trajectory at $t=21$sec.]{\includegraphics[width=4.5cm]{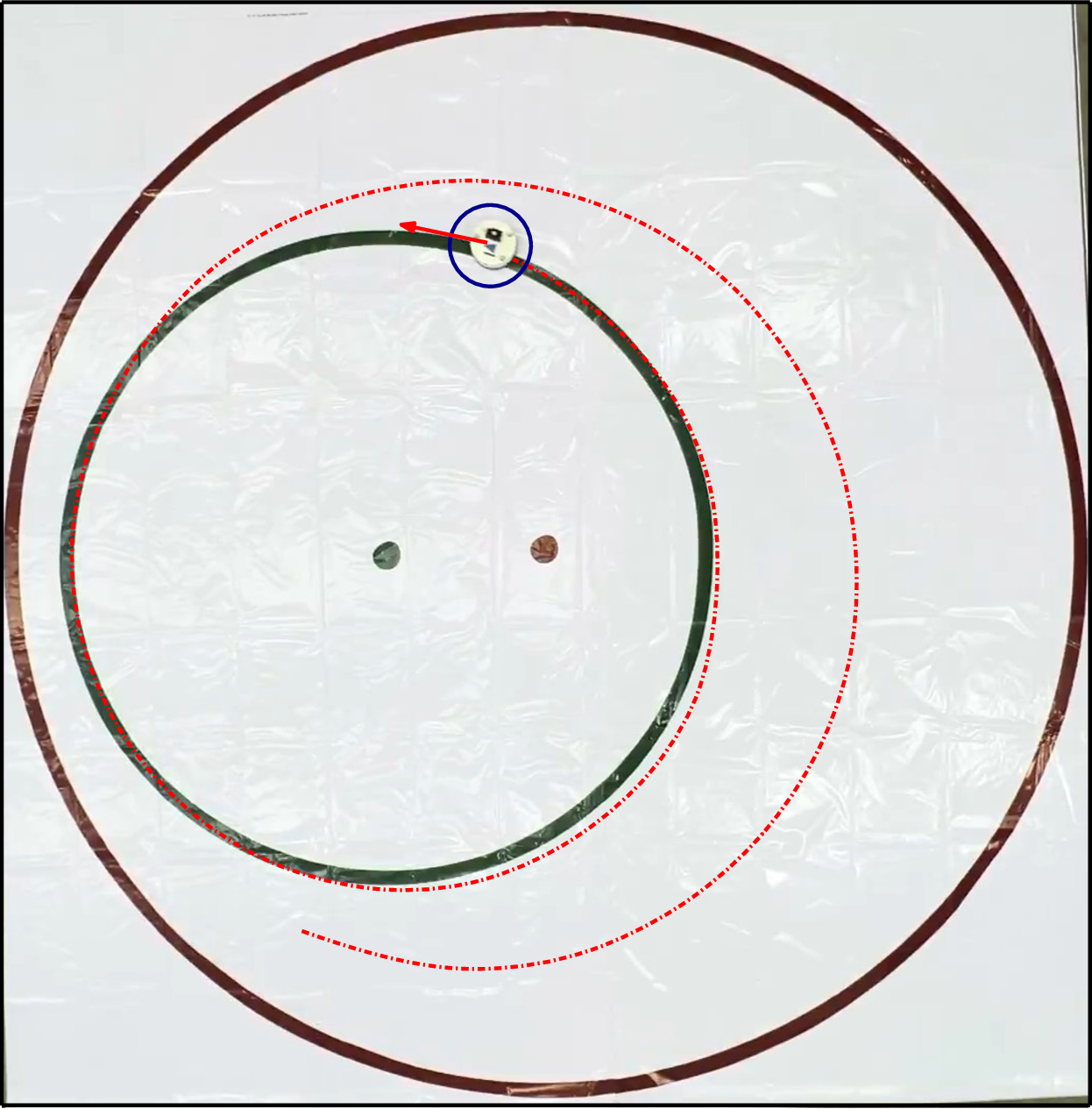}\label{khep_tr}}\hspace{1.2cm}	
	\subfigure[Trajectory in QTM]{\includegraphics[width=4.5cm]{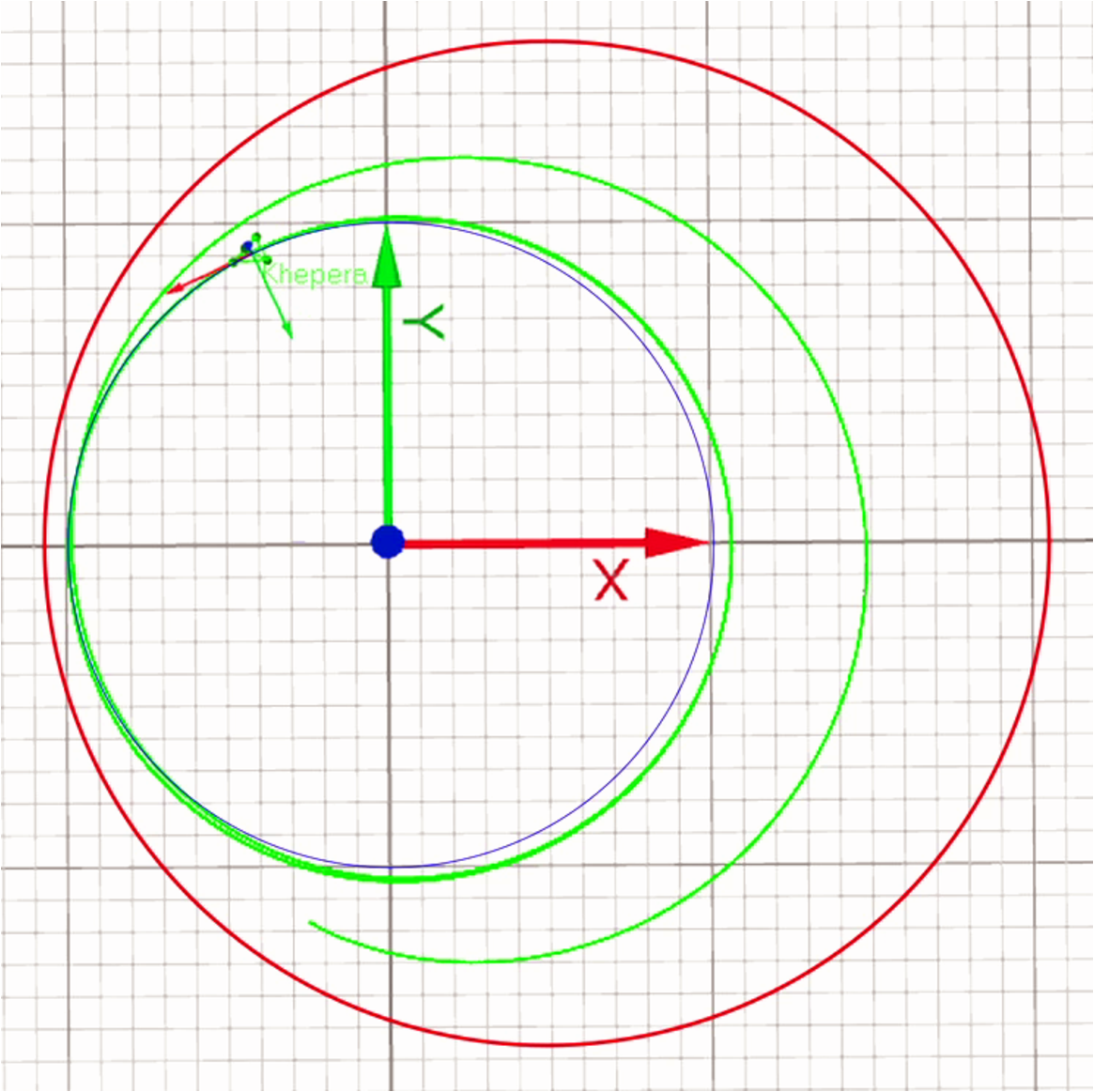}\label{qtm}}\hspace{0.72cm}	
	\subfigure[$\omega$ with $\alpha=\alpha_s$]{\includegraphics[width=4.7cm]{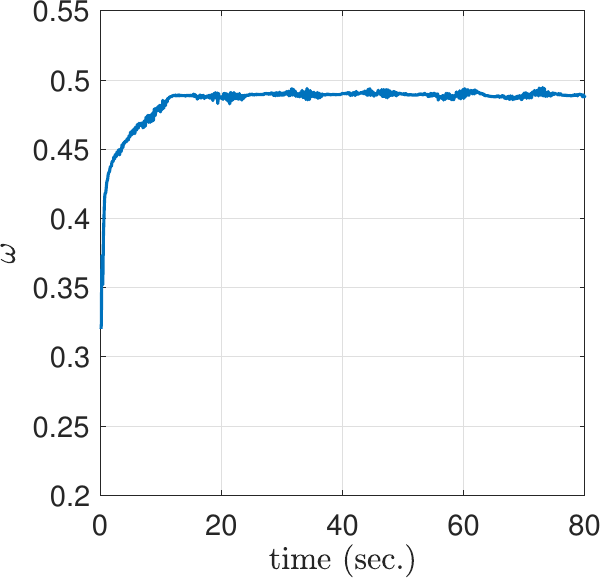}\label{khep_control}}
}
	\caption{Experimental results for Khepera IV robot.}
	\label{plot3}
	\vspace*{-5pt}	
\end{figure*}

\section{Simulation and Experimental Results}\label{section_7}

\subsection{Numerical Simulations}
For illustration, consider the scenario described in Example~\ref{example_1}, where we have nonconcentric circles with $\mu=\sqrt{5/2}$ and $\lambda = 1/2$. Let the robot in the actual plane be characterized by its initial position and heading angle $[x(0), y(0)]^\top = [-0.9, -0.6653]^\top$ and $\theta(0) = -60^\circ$, respectively, and move with constant linear speed $v=1$ m/s. It can be obtained that $|r(0)| = 1.1192$ and $\phi(0) = -143.5274^\circ$. For $\alpha = \alpha_s = 0.5$, $\delta_T = 0.1325$ using \eqref{delta_transformed}. It can be verified that \eqref{initial_conditions} is satisfied, since $\eta_a(0)=-0.045$, $\eta_b(0) = -0.0511$, and $\eta(0) = 0.6428$, with $\sgn(\Delta) = 1$ and $\sigma = 0.5$, resulting in $(\eta_a^2(0)+\eta_b^2(0)  =0.0046)<(\delta_T^2\eta^2(0)=0.0072)$. Using \eqref{position_trnsformed_plane} and \eqref{angle_relation_transformed_plane}, the initial position and heading angle in the transformed plane are obtained as $[\rho_x(0), \rho_y(0)]^\top = [0.0016, -0.6039]^\top$ and $\gamma(0) = 2.3326^\circ$, respectively. It can be further verified that $(|\mathcal{E}(0)|=0.106)<(\delta_T=0.1325)$.

On the other hand, for $\alpha = \alpha_{\ell} = 2$, $\sgn(\Delta)=-1$, $\sigma=-2$ and $\delta_T = 0.4189$. With the same initial conditions in the actual plane, \eqref{initial_conditions} can be verified again, since $\eta_a(0) = -0.5162$, $\eta_b(0) = 0.1741$, and $\eta(0)=1.5526$, resulting in $((\eta_a^2(0)+\eta_b^2(0)) = 0.2968)<(\delta_T^2\eta^2(0) = 0.4229)$. Using \eqref{position_trnsformed_plane} and \eqref{tan_gamma_1}, the position and heading angle in the transformed plane for $\alpha_{\ell}$ are obtained as $[\rho_x(0), \rho_y(0)]^\top = [0.0044, 1.656]^\top$ and $\gamma(0) = 2.0313^\circ$, respectively. One can easily verify that $(|\mathcal{E}(0)| = 0.3509) < (\delta_T = 0.4189)$.

We conducted simulations using control law \eqref{control_law_in_actual_plane_parameters_final} with gain $\kappa = 0.02$ and plotted various attributes in the actual plane as shown in Fig.~\ref{plot1}. It is evident from Fig.~\ref{trj_1} that the robot's trajectory converges to the desired circular orbit $\mathcal{C}$ while being confined within the nonconcentric outer boundary $\mathcal{C}'$. The absolute value of error \eqref{error} converges to the origin, as depicted in Fig.~\ref{error_1}. The control $\omega$ with roots $\alpha_s$ and $\alpha_{\ell}$ is plotted in Fig.~\ref{omega_alpha_s} and Fig.~\ref{omega_alpha_l}, respectively. For both values of $\alpha$, $\omega$ converges to $1$, as desired. Further, Fig.~\ref{plot2} illustrates the robot's motion in the transformed plane and the results are as expected. Please note the difference in the trajectory plots Fig.~\ref{trj_2} (trajectory-constraining) and Fig.~\ref{trj_3} (obstacle-avoidance) obtained for $\alpha_s$ and $\alpha_{\ell}$, respectively. The robot's motion is in the anticlockwise direction in Fig.~\ref{trj_2}, while it is in the clockwise direction in Fig.~\ref{trj_3}, in accordance with Theorem~\ref{thm_sense_of_rotation}.

\subsection{Experimental Results} 
To evaluate the performance of the proposed controller \eqref{control_law_in_actual_plane_parameters_final} in real-world scenarios, we conducted experiments using a Khepera IV\footnote{k-team.com/khepera-iv} differential drive ground robot. The experiment was facilitated by a motion capture (MoCap) validation system comprising overhead cameras. The MoCap system used Qualisys Track Manager (QTM) software for data recording and processing, providing precise feedback on the robot's pose data, including position and heading angle. Control commands were sent to the Khepera IV robot based on measurements obtained from QTM using the Robot Operating System (ROS).

To accommodate the size of Khepera IV robot, we adjusted the initial conditions slightly. The robot commenced its motion from the position $[x(0), y(0)]^\top = [-0.27\text{~m}, -1.15\text{~m}]^\top$ with an initial heading angle of $\theta(0) = -30^\circ$. Given that the Khepera IV is a differential drive robot, we convert the linear and angular velocities of the unicycle model to the individual wheel velocities of the Khepera IV robot as follows:
\begin{equation}\label{khepera}
v_r = v + \left(\frac{d_w}{2}\right) \omega,  \qquad v_{\ell} = v - \left(\frac{d_w}{2}\right) \omega,   
\end{equation}
where $v_r$ and $v_{\ell}$ are the speeds of the right and left wheels, respectively, and $d_w = 10.54$ cm is the distance between the two wheels of the Khepera IV robot. The robot was operated at a constant linear speed of $v = 0.5$ m/s, adhering to its hardware limit for each wheel speed of $0.814$ m/s. Using the control law \eqref{control_law_in_actual_plane_parameters_final}, we derived $v_r$ and $v_{\ell}$ in \eqref{khepera}. Fig.~\ref{khep_tr} shows the top view of the experimental result at $t=21$ seconds where the red dotted line depicts the robot's path. Throughout the experiment, the robot's trajectory was recorded using the QTM software, as depicted in Fig.\ref{qtm}. Here, the $X$-axis and $Y$-axis denote the real and imaginary axes, respectively. Further, Fig.~\ref{khep_control} plots the control law applied during the experiment, demonstrating that the turn rate converges to $\omega\approx0.5$ rad/sec at the steady state.

\begin{figure*}[t!]
	\centering{\hspace*{-0.3cm}
		\subfigure[$\alpha > 0$]{\includegraphics[width=6.0cm]{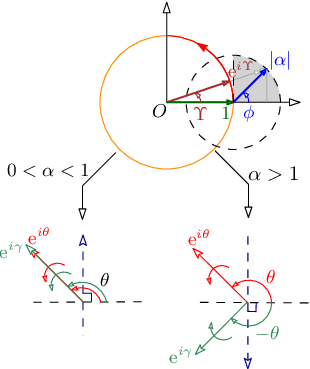}\label{alpha_1st}}\hspace*{1.0cm}
		\subfigure[$-1 < \alpha<  0$]{\includegraphics[width=3.9cm]{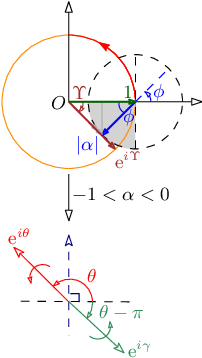}\label{alpha_2nd}}\hspace*{1.0cm}
		\subfigure[$\alpha < -1$]{\includegraphics[width=5.0cm]{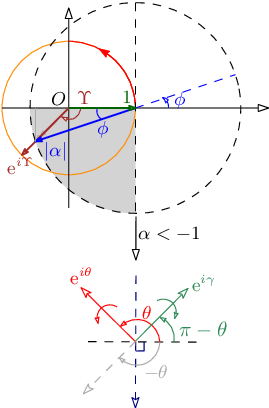}\label{alpha_3rd}}
	}
	\caption{Sense of rotation of the unit vectors ${\rm e}^{i\theta}$ and ${\rm e}^{i\gamma}$ depending on the roots $\alpha$ of \eqref{mobius_roots}.}
	\label{fig_resultant_vector}
	\vspace*{-10pt}	
\end{figure*}

\section{Conclusion and Further Remarks}\label{section_8}
Using the concepts of M\"{o}bius transformation and BLF, we investigated the problem of stabilizing a unicycle robot around a desired circular orbit, while limiting its trajectories within a nonconcentric circular boundary. Such problems find numerous applications in civilian and defense sectors where the motion of a robot is required to be confined within a given geofence. It was shown that the original problem can be equivalently formulated either as a trajectory-constraining problem or an obstacle-avoidance problem in the transformed plane, relying on the roots of \eqref{mobius_roots}. It was proved that a unique control law \eqref{control_law_transformed}, together with the condition $|\mathcal{E}(0)| < \delta_T$ as shown in Fig.~\ref{fig_problem_transformed_plane}, was sufficient to address both the contrasting scenarios in the transformed plane. The key advantage of this transformation lies in converting nonuniform boundary constraints into uniform ones for the robot's motion. Our further efforts have been on establishing a connection between the two control laws $\omega$ and $\Omega$ that was done using both the M\"{o}bius transformation \eqref{position_trnsformed_plane} and its inverse \eqref{inverse_position_in_actual_plane}. In carrying out this analysis, the most crucial relationship was to find a connection between the two heading angles $\theta$ and $\gamma$ in the two planes, which is done in Lemma~\ref{lem_heading_angle_relation_transformed_plane} and Lemma~\ref{lem_heading_angle_relation_actual_plane}, followed by the time-derivatives $\dot{\chi}$ and $\dot{\xi}$ of the transformation-infused angles $\chi$ and $\xi$, as done in Lemma~\ref{lem_chi_dot} and Lemma~\ref{lem_xi_dot}. Finally, we implemented the proposed control approach on the Khepera IV ground robot in a MoCap validation system.  

We would like to highlight an obvious but important fact that while transforming the motion of the robot under the M\"{o}bius transformation \eqref{mobius_transformation}, one should be careful that $\dot{\rho} \neq f(\dot{r})$ and $\mathcal{E} \neq f(e)$. The transformation \eqref{mobius_transformation} is applied only on the position, i.e., $\rho = f(r)$ and other attributes of the robot's motion such as speed and control law must be derived from relation \eqref{position_actual_plane} in the transformed plane. An extension of the problem to the multi-robot framework is an interesting avenue for future research. 


\appendices

\section{Proof of Theorem~\ref{thm_sense_of_rotation}}\label{appendix_A}

\begin{proof}[Proof of Theorem~\ref{thm_sense_of_rotation}]
	To prove the result, we leverage the relation \eqref{e_gamma_final} which connects the unit vectors ${\rm e}^{i\theta}$  and ${\rm e}^{i\gamma}$ along the robot's velocity direction in the actual and the transformed planes, respectively. To facilitate further discussion, we define the unit vector 
	\begin{equation}\label{unit_vector_1+alpha_r}
		\frac{1 + \alpha r}{|1 + \alpha r|} \coloneqq {\rm e}^{i\Upsilon}, \quad  \Upsilon = \arg(1 + \alpha r). 
	\end{equation}
	Using \eqref{unit_vector_1+alpha_r}, \eqref{e_gamma_final} can be expressed as 
	\begin{equation}\label{e_gamma_new}
		{\rm e}^{i\gamma} = \frac{\alpha(1 - \alpha^2)}{|\alpha(1 - \alpha^2)|}{\rm e}^{i(\theta - 2\Upsilon)}.
	\end{equation} 
	Note that \eqref{e_gamma_new} contains two terms $-$ a fraction depending only on $\alpha$ and a unit vector depending on the argument $\theta - 2\Upsilon$. In the following, we analyze both these terms corresponding to the robot's motion in the actual plane and comment on the robot's velocity direction $\gamma$ in the transformed plane. Since our interest is in checking only the sense of rotation, we consider, without loss of generality, the robot moves in the anticlockwise direction on the desired circle $\mathcal{C}: |z| = 1$ in the actual plane and observe its motion in any of the four quadrants. We consider here the first quadrant for simplicity where it holds that $\pi/2 \leq \theta \leq \pi$, $0 \leq \phi \leq \pi/2$ and $\theta - \phi = \pi/2$ (see Fig.~\ref{fig_problem}). Further, note that $1 + \alpha r$ is the resultant of vectors $1$ and $\alpha r$. Geometrically, this can be obtained by drawing a circle of radius $|\alpha r| = |\alpha|$ (since $|r| = 1$) with center at $1 + i0$, as shown in Fig.~\ref{fig_resultant_vector} (the unit circle is plotted in orange). The unit vector ${\rm e}^{i\Upsilon}$ along $1 + \alpha r$ is marked in brown in Fig.~\ref{fig_resultant_vector}. It is clear that the argument $\Upsilon$ of $1 + \alpha r$ varies with $\alpha$. For the robot's motion in the first quadrant (as shown by the red circular arc in Fig.~\ref{fig_resultant_vector}), we now consider the following four cases depending on the roots $\alpha$ of \eqref{mobius_roots}: 
	\begin{enumerate}[leftmargin=*]
		\item[(I)] If $0 < \alpha < 1$. In this case, vectors $1 + \alpha r$ lie in the shaded region in the 1st quadrant as depicted in Fig.~\ref{alpha_1st} and the argument of $1 + \alpha r$ is given by: 
		\begin{equation}\label{argument_1st}
			\Upsilon = \arctan\left[\frac{\alpha \sin \phi}{1 + \alpha \cos\phi}\right].
		\end{equation}
		From \eqref{argument_1st}, one can easily obtain $\Upsilon$ for the critical values of $\alpha$. If $\alpha \to 0$, $\Upsilon \to 0$. If  $\alpha \to 1$, $\Upsilon \to \arctan\left[\frac{\sin \phi}{1 + \cos\phi}\right] = \phi/2$, using the half-angle formula. Combining both cases, $0 < \Upsilon < \phi/2$. Consequently, it holds that $\theta - \phi < \theta - 2\Upsilon < \theta$. Using $\theta - \phi = \pi/2$, the preceding inequality becomes $\pi/2 < \theta - 2\Upsilon < \theta$. Since the fraction $\frac{\alpha(1 - \alpha^2)}{|\alpha(1 - \alpha^2)|}$ in \eqref{e_gamma_new} evaluates to $1$ in this case, ${\rm e}^{i\gamma} = {\rm e}^{i(\theta - 2\Upsilon)}$. Now, it immediately follows that $\pi/2 < \gamma < \theta$. This implies the counterclockwise rotation of the unit vector ${\rm e}^{i\gamma}$ in the transformed plane for the increasing $\theta$ (see Fig.~\ref{alpha_1st}, left figure). 
		
		\item[(II)] If $\alpha > 1$. Again, $\Upsilon$ is given by \eqref{argument_1st}. If $\alpha \to \infty$, $\Upsilon \to \phi$. Consequently, it holds that $\phi/2 < \Upsilon < \phi$, and hence, $\theta - 2\phi < \theta - 2\Upsilon < \theta - \phi$. Using $\theta - \phi = \pi/2$, the preceding inequality results in $\pi - \theta < \theta - 2\Upsilon < \pi/2$. Since the fraction $\frac{\alpha(1 - \alpha^2)}{|\alpha(1 - \alpha^2)|}$ in \eqref{e_gamma_new} evaluates to $-1$ in this case, ${\rm e}^{i\gamma} = -{\rm e}^{i(\theta - 2\Upsilon)} \implies {\rm e}^{i(\pi + \gamma)} = {\rm e}^{i(\theta - 2\Upsilon)}$. Alternatively, it follows that $\pi - \theta < \pi + \gamma < \pi/2 \implies -\theta < \gamma < -\pi/2$. This implies the clockwise rotation of ${\rm e}^{i\gamma}$ in the transformed plane for the increasing $\theta$ (see Fig.~\ref{alpha_1st}, right figure).
		
		\item[(III)] If $-1 < \alpha < 0$. In this case, vectors $1 + \alpha r$ lie in the shaded region in the 4th quadrant as shown in Fig.~\ref{alpha_2nd} and  
		\begin{equation}\label{argument_2nd}
			\Upsilon = -\arctan\left[\frac{|\alpha| \sin \phi}{1 - |\alpha| \cos\phi}\right].
		\end{equation}
		If $\alpha \to 0$, $\Upsilon \to 0$. If  $\alpha \to -1$, $\Upsilon \to -\arctan\left[\frac{\sin \phi}{1 - \cos\phi}\right] = -(\pi - \phi)/2$. This implies that $-(\pi - \phi)/2 < \Upsilon < 0$. Consequently, it can be concluded that $\theta < \theta - 2\Upsilon < 3\pi/2$, following the same steps as in the previous cases. Since $\frac{\alpha(1 - \alpha^2)}{|\alpha(1 - \alpha^2)|}$ in \eqref{e_gamma_new} evaluates to $-1$ in this case, ${\rm e}^{i\gamma} = -{\rm e}^{i(\theta - 2\Upsilon)} \implies {\rm e}^{i(\pi + \gamma)} = {\rm e}^{i(\theta - 2\Upsilon)}$. Alternatively, it follows that $\theta < \pi + \gamma < 3\pi/2 \implies \theta - \pi < \gamma < \pi/2$. This implies the clockwise rotation of ${\rm e}^{i\gamma}$ in the transformed plane for the increasing $\theta$ (see Fig.~\ref{alpha_2nd}).

		\item[(IV)] If $\alpha < -1$. In this case, the fraction $\frac{\alpha(1 - \alpha^2)}{|\alpha(1 - \alpha^2)|}$ in \eqref{e_gamma_new} evaluates to $1$. However, vectors $1 + \alpha r$ may lie in both 3rd and 4th quadrants (and hence, ${\rm e}^{i\Upsilon}$), depending on the magnitude of $\alpha$ (see Fig.~\ref{alpha_3rd}). We now consider the following two sub-cases:
		\begin{itemize}[leftmargin=*]
			\item Let ${\rm e}^{i\Upsilon}$ lies in the 3rd quadrant. In this situation, it must hold that $|\alpha|\cos\phi > 1 \implies \phi \in \left[0, \arccos\left(\frac{1}{|\alpha|}\right)\right)$ and
			\begin{equation}\label{argument_3rd}
				\Upsilon = -\left[\pi - \arctan\left[\frac{|\alpha| \sin \phi}{|\alpha| \cos\phi - 1}\right]\right].
			\end{equation}
			If  $\alpha \to -1$, $\Upsilon \to -\left[\pi - \arctan\left[\frac{\sin \phi}{\cos\phi - 1} \right]\right] = -(\pi -  (\phi - \pi)/2) = -(3\pi - \phi)/2$. If $\alpha \to -\infty$, $\Upsilon \to -(\pi - \phi)$. In other words, $-(\pi - \phi) < \Upsilon < -(3\pi - \phi)/2$. Consequently, it holds that $7\pi/2 < \theta - 2\Upsilon < 3\pi - \theta$, following the same steps as above. Alternatively, $3\pi/2 < \theta - 2\Upsilon -2\pi < \pi - \theta$. This implies the clockwise rotation of ${\rm e}^{i\gamma}$ in the transformed plane for the increasing $\theta$ (see Fig.~\ref{alpha_3rd}).
			
			\item Let ${\rm e}^{i\Upsilon}$ lies in the 4th quadrant. In this situation, it must hold that $|\alpha|\cos\phi < 1 \implies \phi \in \left(\arccos\left(\frac{1}{|\alpha|}\right), \frac{\pi}{2}\right]$ and $\Upsilon$ is given by \eqref{argument_2nd}. If  $\alpha \to -1$, $\Upsilon \to -(\pi - \phi)/2$. If $\alpha \to -\infty$, $\Upsilon \to -(\pi - \phi)$. In other words, $-(\pi - \phi) < \Upsilon < -(\pi - \phi)/2$. Consequently, it holds that $3\pi/2 < \theta - 2\Upsilon < 3\pi - \theta$. Alternatively, $-\pi/2 < \theta - 2\Upsilon -2\pi < \pi - \theta$. This again implies the clockwise rotation of ${\rm e}^{i\gamma}$ in the transformed plane for the increasing $\theta$ (see Fig.~\ref{alpha_3rd}).
		\end{itemize}
	\end{enumerate}
	Here, Case (I) and Case (III) refer to $\alpha = \alpha_s$ and the rest two cases are for $\alpha = \alpha_{\ell}$. This concludes the proof. 	
\end{proof}

\section{Proof of Lemma~\ref{lem_chi_dot}}\label{appendix_B}

Before discussing the proof of Lemma~\ref{lem_chi_dot}, we first present the following intermediate result. 

\begin{lem}[Relationship between $\theta$ and $\dot{\phi}$]\label{lem_phi_dot}
	The actual heading angle $\theta$ and the time-derivative $\dot{\phi}$ of the angle $\phi$ of the position vector $r$ in \eqref{position_actual_plane} are related as
	\begin{align}\label{phi_dot}
		\dot{\phi} = \frac{v\sin(\theta-\phi)}{|r|}.
	\end{align}
\end{lem}

\begin{proof}
	Following the fact that $\langle r, r \rangle = |r|^2$ and $\sqrt{\langle r, r \rangle} = |r|$, the	time-derivative of \eqref{position_actual_plane} can be obtained using chain rule as:
	\begin{equation}\label{r_dot}
		\dot{r} = \left(\frac{d}{dt} \sqrt{\langle r, r \rangle} \right) {\rm e}^{i\phi} + i|r|{\rm e}^{i\phi}\dot{\phi}.
	\end{equation}
	Here, the time-derivatives of $\langle r, r \rangle$ and $\sqrt{\langle r, r \rangle}$, can be obtained as follows:
	\begin{align}\label{mag_r_std}
		\nonumber \frac{d}{dt}\langle r,r \rangle &= \langle r, \dot{r} \rangle + \langle \dot{r}, r \rangle = 2 \langle r, \dot{r} \rangle\\
		&= 2|r|v\langle {\rm e}^{i\phi}, {\rm e}^{i\theta} \rangle = 2|r|v\cos(\theta-\phi).
	\end{align}
	Similarly,
	\begin{align}\label{mag_r_td}
		\frac{d}{dt}\sqrt{\langle r,r\rangle} &= \frac{2\langle  r , \dot{r} \rangle}{2\sqrt{\langle r,r\rangle}} = \frac{\langle  r , \dot{r} \rangle}{\sqrt{\langle r,r \rangle}} = v\cos(\theta-\phi).
	\end{align}
	Substituting \eqref{mag_r_td} into \eqref{r_dot}, we have
	\begin{align*}
		\dot{r} = v\cos(\theta-\phi){\rm e}^{i\phi} + i|r|{\rm e}^{i\phi}\dot{\phi},
	\end{align*}
	which, upon multiplication by ${\rm e}^{-i\phi}$ on both sides, yields
	\begin{equation*}
		\dot{r}{\rm e}^{-i\phi} = v\cos(\theta-\phi) + i|r|\dot{\phi}.
	\end{equation*}
	Substituting $\dot{r}$ from \eqref{dyn} and comparing the real and imaginary parts, we get \eqref{phi_dot}. 
\end{proof}

\begin{proof}[Proof of Lemma~\ref{lem_chi_dot}]
From \eqref{chi}, the time-derivative of $\chi$ can be obtained as:
\begin{equation} \label{chi_dot}
\dot{\chi} = \frac{f_1\dot{f}_2 - f_2\dot{f}_1}{f_1^2+f_2^2} \coloneqq \frac{\chi_n}{\chi_d}.
\end{equation}
Here, the denominator $\chi_d$ is obtained, using \eqref{f1f2}, as: 
\begin{align}\label{f1f2_square_new}
	\nonumber &f_1^2+f_2^2  = (1+2\alpha|r|\cos\phi+\alpha^2|r|^2\cos2\phi)^2\\
	\nonumber & \qquad +(2\alpha|r|\sin\phi +\alpha^2|r|^2\sin 2 \phi)^2\\
	\nonumber & = (1+2\alpha|r|\cos\phi)^2+\alpha^4|r|^4(\cos^22\phi+\sin^22\phi)\\
	\nonumber & \qquad +2\alpha^2|r|^2\cos2\phi+4\alpha^2|r|^2\sin^2\phi\\
	\nonumber & \qquad +4\alpha^3|r|^3(\cos\phi\cos2\phi+\sin\phi\sin2\phi)\\
	\nonumber & = 1+\alpha^4|r|^4+4\alpha^2|r|^2+4\alpha|r|\cos\phi\\
	\nonumber & \qquad +2\alpha^2|r|^2\cos2\phi+4\alpha^3|r|^3\cos\phi\\
	\nonumber & = 1+\alpha^4|r|^4+4\alpha^2|r|^2+4\alpha|r|\cos\phi\\
	\nonumber & \qquad +2\alpha^2|r|^2(2\cos^2\phi-1)+4\alpha^3|r|^3\cos\phi\\
	\nonumber & =(1+\alpha^2|r|^2)^2+4\alpha|r|\cos\phi(1+\alpha^2|r|^2)+4\alpha^2|r|^2\cos^2\phi\\
	& =(1+\alpha^2|r|^2+2\alpha|r|\cos\phi)^2=|1+\alpha r|^4. 
\end{align}
Further, the numerator $\chi_n$ in \eqref{chi_dot} can be derived exploiting \eqref{mag_r_std} and \eqref{mag_r_td}. Using these relations, the time-derivative of $f_1$ and $f_2$ is obtained from \eqref{f1f2} as:
\begin{subequations}\label{f1f2_dot}
\begin{align}
\nonumber \dot{f}_1 & = 2\alpha v\cos(\theta-\phi)(\cos\phi+\alpha|r|\cos2\phi)\\
& \qquad -2\alpha|r|\dot{\phi}(\sin\phi+\alpha|r|\sin2\phi),\\
\nonumber \dot{f}_2 & = 2\alpha v\cos(\theta-\phi)(\sin\phi+\alpha|r|\sin2\phi)\\
& \qquad +2\alpha|r|\dot{\phi}(\cos\phi+\alpha|r|\cos2\phi).
\end{align}	
\end{subequations}
Again using \eqref{f1f2}, \eqref{f1f2_dot} can be written as
\begin{align*}
\dot{f}_1 & = 2\tau (f_1-1-\alpha|r|\cos\phi)-2\dot{\phi}(f_2-\alpha|r|\sin\phi),\\
\dot{f}_2 & = 2\tau (f_2 - \alpha|r|\sin\phi) + 2\dot{\phi}(f_1-1-\alpha|r|\cos\phi).
\end{align*}
Furthermore, $f_1\dot{f}_2$ and $f_2\dot{f}_1$ are obtained as
\begin{subequations}\label{f1df2}
	\begin{align}
		f_1\dot{f}_2 & = f_1[2\tau (f_2 - \alpha|r|\sin\phi) + 2\dot{\phi}(f_1-1-\alpha|r|\cos\phi)]\\
		f_2\dot{f}_1 & = f_2[2\tau (f_1-1-\alpha|r|\cos\phi)-2\dot{\phi}(f_2-\alpha|r|\sin\phi)],
	\end{align}	
\end{subequations}
where we denote by
\begin{equation}\label{tau_ap}
\tau \coloneqq \frac{v\cos(\theta -\phi)}{|r|}.
\end{equation}
Now, substituting \eqref{f1df2} into \eqref{chi_dot}, $\chi_n$ is given by
\begin{align}\label{chi_n}
\nonumber	\chi_n & = f_1[2\tau (f_2 - \alpha|r|\sin\phi) + 2\dot{\phi}(f_1-1-\alpha|r|\cos\phi)]\\
\nonumber	&  \qquad - f_2[2\tau (f_1-1-\alpha|r|\cos\phi)-2\dot{\phi}(f_2-\alpha|r|\sin\phi)]\\
\nonumber	& = 2\dot{\phi}(f_1^2-f_1-f_1\alpha|r|\cos\phi+f_2^2-f_2\alpha|r|\sin\phi)\\
\nonumber	& \qquad + 2\tau(f_1 f_2 - f_1\alpha|r|\sin\phi - f_1 f_2 + f_2 + f_2\alpha|r|\cos\phi)\\
\nonumber	& = 2\dot{\phi}[f_1^2+f_2^2-f_1-\alpha|r|\underbrace{(f_1\cos\phi+f_2\sin\phi)}_{\vartheta_1}]\\
& \qquad + 2\tau(f_2-\alpha|r|\underbrace{(f_1\sin\phi-f_2\cos\phi)}_{\vartheta_2}),
\end{align}
where $\vartheta_1$ and $\vartheta_2$ can be simplified using \eqref{f1f2} as:
\begin{align}\label{vartheta1}
\nonumber	\vartheta_1 &= \cos\phi + 2\alpha|r|\cos^2\phi + \alpha^2|r|^2\cos\phi\cos2\phi\\
\nonumber	& \qquad +2\alpha|r|\sin^2\phi + \alpha^2|r|^2\sin\phi\sin2\phi\\
\nonumber	& = \cos\phi+2\alpha|r|+\alpha^2|r|^2\cos\phi \\
	& =2\alpha|r|+(1+\alpha^2|r|^2)\cos\phi,
\end{align}
and
\begin{align}\label{vartheta2}
\nonumber \vartheta_2 & = \sin\phi + 2\alpha|r|\cos\phi\sin\phi+\alpha^2|r|^2\cos2\phi\sin\phi\\
\nonumber	& \qquad -2\alpha|r|\cos\phi\sin\phi-\alpha^2|r|^2\cos\phi\sin\phi\\
	& = \sin\phi-\alpha^2|r|^2\sin\phi = (1-\alpha^2|r|^2)\sin\phi.
\end{align}
Now, using \eqref{vartheta1} and \eqref{vartheta2}, \eqref{chi_n} results in
\begin{align}
\nonumber \chi_n	& = 2\dot{\phi}[f_1^2+f_2^2-f_1-\alpha|r|(2\alpha|r|+(1+\alpha^2|r|^2)\cos\phi)]\\
	& \qquad + 2\tau[f_2-\alpha|r|(1-\alpha^2|r|^2)\sin\phi],
\end{align}
that, on further substitution for $f_1^2+f_2^2$ from \eqref{f1f2_square_new} and using \eqref{f1f2} again, yields
\begin{align}\label{chi_num_new}
\nonumber	\chi_n	& = 2\dot{\phi}[(1+\alpha^2|r|^2+2\alpha|r|\cos\phi)^2-1-2\alpha|r|\cos\phi\\
\nonumber	&\qquad -\alpha^2|r|^2\cos2\phi-\alpha|r|(2\alpha|r|+(1+\alpha^2|r|^2)\cos\phi)]\\
\nonumber	&\qquad + 2\tau[2\alpha|r|\sin\phi +\alpha^2|r|^2\sin 2 \phi-\alpha|r|(1-\alpha^2|r|^2)\sin\phi]\\
\nonumber	& = 2\dot{\phi}[(1+\alpha^2|r|^2+2\alpha|r|\cos\phi)^2 - (1+\alpha^2|r|^2+2\alpha|r|\cos\phi)\\
\nonumber	&\qquad -\alpha|r|\cos\phi(1+\alpha^2|r|^2+2\alpha|r|\cos\phi)]\\
\nonumber	&\qquad +2\tau\alpha|r|\sin\phi(1+\alpha^2|r|^2+2\alpha|r|\cos\phi)\\
\nonumber	& =2(1+\alpha^2|r|^2+2\alpha|r|\cos\phi)[\dot{\phi}(1+\alpha^2|r|^2\\
	        & \qquad +2\alpha|r|\cos\phi-1-\alpha|r|\cos\phi) +\tau\alpha|r|\sin\phi].
\end{align}
Next, substituting \eqref{f1f2_square_new} and \eqref{chi_num_new} into \eqref{chi_dot}, gives 
\begin{align*}
\dot{\chi} &= \frac{2(1+\alpha^2|r|^2+2\alpha|r|\cos\phi)\dot{\phi}(\alpha^2|r|^2+\alpha|r|\cos\phi)}{(1+\alpha^2|r|^2+2\alpha|r|\cos\phi)^2}\\
	&~~~+\frac{2(1+\alpha^2|r|^2+2\alpha|r|\cos\phi)\tau\alpha|r|\sin\phi}{(1+\alpha^2|r|^2+2\alpha|r|\cos\phi)^2}\\
	&=\frac{2[\dot{\phi}(\alpha^2|r|^2+\alpha|r|\cos\phi)+\tau\alpha|r|\sin\phi]}{1+\alpha^2|r|^2+2\alpha|r|\cos\phi}.
\end{align*}
Further, substitution of $\tau$ from \eqref{tau_ap}, yields
\begin{align*}
\dot{\chi} & = \frac{2[\dot{\phi}\alpha|r|(\alpha|r|+\cos\phi)+\alpha v \cos(\theta-\phi)\sin\phi]}{1+\alpha^2|r|^2+2\alpha|r|\cos\phi},
\end{align*}
that, after using \eqref{phi_dot}, leads to
\begin{align*}
\dot{\chi} & = \frac{2[\alpha v \sin(\theta-\phi)(\alpha|r|+\cos\phi)+\alpha v \cos(\theta-\phi)\sin\phi]}{1+\alpha^2|r|^2+2\alpha|r|\cos\phi}\\
& = \frac{2\alpha v[\sin\theta+\alpha|r|\sin(\theta-\phi)]}{1+\alpha^2|r|^2+2\alpha|r|\cos\phi}\\
& = \frac{2\alpha v[\sin\theta+\alpha|r|\sin(\theta-\phi)]}{|1+\alpha r|^2}.
\end{align*}
This concludes the proof.
\end{proof}


\bibliographystyle{IEEEtran}
\bibliography{References}

\end{document}